\documentclass[11pt,letter]{article}
\usepackage[margin=1in]{geometry}
\usepackage{pdfpages}
\newcommand{\dtlinkcolor}{{0.8 0.8 1}} 
\usepackage[hyperindex=true,pdfpagemode=UseOutlines,bookmarksnumbered=true,bookmarksopen=true,bookmarksopenlevel=2,pdfstartview=FitH,pdfborder={0 0 1},linkbordercolor=\dtlinkcolor,citebordercolor=\dtlinkcolor,urlbordercolor=\dtlinkcolor,pagebordercolor=\dtlinkcolor]{hyperref}
\usepackage[utf8]{inputenc} 
\usepackage[T1]{fontenc}    
\usepackage{palatino}       
\usepackage{url}            
\usepackage{booktabs}       
\usepackage{amsfonts}       
\usepackage{nicefrac}       
\usepackage{microtype}      
\usepackage{xcolor}         
\usepackage{wrapfig}
\usepackage{caption}
\usepackage{subcaption}
\usepackage{algorithm}
\usepackage{natbib}

\usepackage[utf8]{inputenc} 
\usepackage[T1]{fontenc}    
\usepackage{hyperref}       
\usepackage{url}            
\usepackage{enumitem}
\usepackage{amsfonts}       
\usepackage{nicefrac}       
\usepackage{hanch_sty}
\usepackage{ifthen}
\newboolean{showcomments}
\setboolean{showcomments}{false}
\usepackage{color}
\usepackage{todonotes}

\definecolor{bleudefrance}{rgb}{0.19, 0.55, 0.91}
\definecolor{ao(english)}{rgb}{0.0, 0.5, 0.0}

\newcommand{\addcite}[0]{\ifthenelse{\boolean{showcomments}}
{\textcolor{purple}{(add cite(s)) }}{}}%

\newcommand{\addref}[0]{\ifthenelse{\boolean{showcomments}}
{\textcolor{purple}{(add ref) }}{}}%

\newcommand{\enrique}[1]{  \ifthenelse{\boolean{showcomments}}
{\todo[inline,color=bleudefrance]{Enrique: #1}}{}}
\newcommand{\rene}[1]{  \ifthenelse{\boolean{showcomments}}
{\todo[inline,color=cyan]{Ren\'e: #1}}{}}
\newcommand{\emmargin}[1]{\ifthenelse{\boolean{showcomments}}{\marginpar{\color{bleudefrance}\tiny EM: #1}}{}}
\newcommand{\hmmargin}[1]{\ifthenelse{\boolean{showcomments}}{\marginpar{\color{orange}\tiny HM: #1}}{}}
\newcommand{\hancheng}[1]{  \ifthenelse{\boolean{showcomments}}
{\todo[inline,color=orange]{Hancheng: #1}}{}}

\newcommand{\hl}[1]{\ifthenelse{\boolean{showcomments}}
{\textcolor{red}{#1}}{#1}}

\usepackage{bbold,amsmath,amsthm,amssymb}
\usepackage{multirow}

\newcommand{\rank}{\mathrm{rank}}

\newcommand{\sign}{\mathrm{sign}}
\newcommand{\act}{\sigma}
\newcommand{\var}{\alpha}
\usepackage[normalem]{ulem}
\newcommand{\myparagraph}[1]{\smallskip\noindent\textbf{#1.}}

\newtheorem{theorem}{Theorem}

\newtheorem{innercustomthm}{Theorem}
\newenvironment{customthm}[1]
  {\renewcommand\theinnercustomthm{#1}\innercustomthm}
  {\endinnercustomthm}
 \newtheorem{innercustomlem}{Lemma}
\newenvironment{customlem}[1]
  {\renewcommand\theinnercustomlem{#1}\innercustomlem}
  {\endinnercustomlem}
  
   \newtheorem{innercustomprop}{Proposition}
\newenvironment{customprop}[1]
  {\renewcommand\theinnercustomprop{#1}\innercustomprop}
  {\endinnercustomprop}
  
\newtheorem{lemma}{Lemma}

\newtheorem{proposition}{Proposition}

\newtheorem{conjecture}{Conjecture}
\newtheorem{remark}{Remark}

\newtheorem*{claim}{Claim}

\title{Can Implicit Bias Imply Adversarial Robustness?}

\author{%
  Hancheng Min and Ren\'e Vidal \\
  \\
  Center for Innovation in Data Engineering and Science (IDEAS)
  \\
  University of Pennsylvania 
  \\
}
\date{}
\begin{document}

\maketitle

\begin{abstract}
    The implicit bias of gradient-based training algorithms has been considered mostly beneficial as it leads to trained networks that often generalize well. However, \citet{frei2023the} show that such implicit bias can harm adversarial robustness. Specifically, they show that if the data consists of clusters with small inter-cluster correlation, a shallow (two-layer) ReLU network trained by gradient flow generalizes well, but it is not robust to adversarial attacks of small radius. Moreover, this phenomenon occurs despite the existence of a much more robust classifier that can be explicitly constructed from a shallow network. 
    In this paper, we extend recent analyses of neuron alignment to show that a shallow network with a polynomial ReLU activation (pReLU) trained by gradient flow not only generalizes well but is also robust to adversarial attacks. Our results highlight the importance of the interplay between data structure and architecture design in the implicit bias and robustness of trained networks.
\end{abstract}

\section{Introduction}
Behind the success of deep neural networks in many application domains lies their vulnerability to \emph{adversarial attacks}, i.e., small and human-imperceptible perturbations to the input data. Such a phenomenon was observed in the seminal paper of \citet{szegedy2014intriguing} and has motivated a large body of work on building defenses against such attacks \citep{shafahi2019adversarial, papernot2016distillation,wong2019fast, guo2018countering, cohen2019certified, levine2020randomized, yang2020randomized, sulam2020adversarial,kinfu2022analysis}. 

However, many defense strategies have been shown to fail against new adaptive attacks~\citep {athalye2018obfuscated, carlini2019evaluating}, and understanding these failures seems to be a fundamental challenge. For example, \citet{fawzi2018adversarial, dohmatob2019generalized, shafahi2018adversarial} show the non-existence of robust classifiers for certain data distributions. Recently, \citet{pal2023adversarial} show that having a data distribution that concentrates within a small-volume subset of the ambient space is necessary for the existence of a robust classifier. These results highlight the importance of understanding and exploiting data structure in the process of finding classifiers with certified robustness, yet almost none of the existing defense strategies do so. 

Besides data distribution, additional issues arise from the training algorithms. For example, \citet{vardi2022margin} and \citet{frei2023the} show that when the data consists of clusters with small inter-cluster correlation, a shallow (two-layer) ReLU network trained by gradient flow generalizes well but fails to be robust against adversarial attacks of small radius, despite the existence of a much more robust classifier that can be explicitly constructed from a shallow network. 
This study unveils new challenges in the search for robust classifiers: Even if we know robust classifiers exist for certain data distribution, the \emph{implicit bias} from our training algorithm (the choice of network architecture, optimization algorithm, etc.) may prevent us from finding it.

\begin{figure*}[ht]
  \centering
  \vspace{-0.3cm}
  \includegraphics[width=\linewidth]{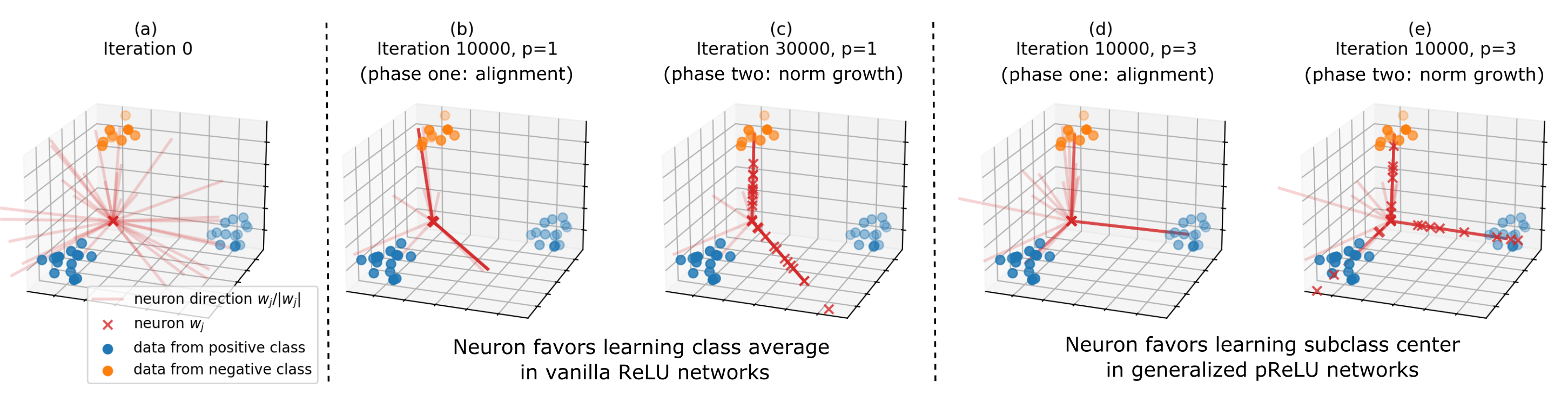}
  \vspace{-0.7cm}
  \caption{Visualizing the training of a pReLU network under small initialization (Details explained in later sections). The dataset has its positive class sampled from two subclasses. \textbf{(a)} At initialization, all neurons have small norms and point toward random directions; \textbf{When $p=1$} (vanilla ReLU network), \textbf{(b)} During the alignment phase, the neuron directions are aligned with either of the average class centers $\bar{\boldsymbol{\mu}}_+$ and $\bar{\boldsymbol{\mu}}_-$; \textbf{(c)} During the second phase, neurons keep the alignment with $\bar{\boldsymbol{\mu}}_+$ and $\bar{\boldsymbol{\mu}}_-$ while growing their norms; \textbf{When $p=3$}, \textbf{(d)} neurons learn subclass centers during alignment phase and \textbf{(e)} keep the alignment in the second phase. Note: the neurons pointing toward directions other than class/subclass centers are not activated by any data point and have small norms throughout training.}
  \label{fig_training_vis}
  \vspace{-0.4cm}
\end{figure*}

\myparagraph{Paper contributions} In this paper, we show that under the same setting studied in~\citet{frei2023the}, the implicit bias of gradient flow that leads to non-robust networks can be altered to favor robust networks by modifying the ReLU activation. Specifically, we consider a data distribution consisting of a mixture of $K$ Gaussians, referred to as \emph{subclasses}, which have small inter-subclass correlation and are grouped into two \emph{superclasses/classes}.
When training a two-layer binary classification network, we show, with formal theorems, and with conjectures validated experimentally, that (also illustrated in Figure \ref{fig_training_vis}):
\begin{itemize}[leftmargin=9pt,topsep=0pt,parsep=0pt]
    \item If the activation is a ReLU, neurons (rows of the first layer weight matrix) tend to learn only the average direction of each class, leading to a classifier that generalizes well on clean data, but is vulnerable to an adversarial attack with $l_2$ radius $\mathcal{O}\lp \frac{1}{\sqrt{K}}\rp$, i.e. the trained network is non-robust with many subclasses. This leads to a new neural alignment perspective on the nonrobustness of ReLU networks identified by \citet{frei2023the}.
    \item If the activation is replaced by a novel polynomial ReLU activation, proposed based on recent advances in understanding the neuron alignment in shallow networks, neurons tend to learn the direction of each subclass center, leading to a classifier that generalizes well on clean data and can sustain any adversarial attack with $\mathcal{O}\lp 1\rp$ radius.
\end{itemize}

Our analysis \textbf{(1)} highlights the importance of the interplay between data structure and network architecture 
in determining the robustness of the trained network, \textbf{(2)} explains how the implicit bias (regularization) of training a ReLU network fails to exploit the data structure and leads to non-robust networks, and \textbf{(3)} shows how the issue is resolved by using a polynomial ReLU activation function. Moreover, numerical experiments on real datasets show that shallow networks with our generalized ReLU activation functions are much more robust than those with a ReLU activation.

\myparagraph{Relation to existing analysis on implicit bias of neural networks} Our discussion is theory-centric. It builds upon the analysis of the implicit bias of training algorithms, but it is significantly different from prior theoretical analyses. The implicit bias of training algorithms has been studied extensively over past years for various architectures, including both linear networks \cite{saxe2014exact,gunasekar2017implicit,ji2019gradient, woodworth2020kernel,mtvm21icml,stoger2021small,jacot2021saddle,wang2023implicit}, and nonlinear networks \cite{lyu2019gradient,ji2020directional,maennel2018gradient,chizat20a,boursier2022gradient,min2023early,wang2023understanding, frei2022implicit, frei2023the, kumar2024directional, abbe2023transformers, ataee2023max}. Moreover, this phenomenon has been studied from different perspectives, including max-margin \cite{lyu2019gradient, chizat20a,ataee2023max}, min-norm \cite{gunasekar2017implicit,mtvm21icml}, sparsity/low-rankness \cite{saxe2014exact,woodworth2020kernel,wang2023implicit,abbe2023transformers}, and alignment \cite{maennel2018gradient,min2023early,kumar2024directional}. While most works consider these implicit biases beneficial for the success of neural networks in practice, and some existing work ~\citep{faghri2021bridging} has even shown that such biases help robustness in the cases of linear regression; few works \cite{frei2023the,boursier2024early} discuss the potential harm caused by such biases. Our work takes one step further by proposing fixes to the harms identified by prior work~\citep{frei2023the}, which sheds light on the potential of using deep learning theory to not only understand but also improve neural networks in practice.

\hl{\myparagraph{Paper organization} Our main formal results regarding adversarial robustness will be shown (in Section \ref{sec:main_adv}) for explicitly constructed classifiers that can be realized by shallow networks with different activation functions. Then we connect these robustness results to those of actual trained shallow networks by a conjecture that shallow networks, when initialized properly, can learn such constructions via gradient flow training. The rationale behind this conjecture is carefully explained in Section \ref{sec:theoretical_analsysis}, together with a preliminary theoretical analysis that supports this conjecture. Lastly, we empirically verify our conjecture in Section \ref{sec:num} via experiments on synthetic data, and then we conduct numerical experiments on real datasets, showing that our proposed new activation improves the robustness of shallow networks.}

\textbf{Notation}: We denote the the inner product between vectors $\boldsymbol{x}$ and $\boldsymbol{y}$ by $\lan \boldsymbol{x},\boldsymbol{y}\ran=\boldsymbol{x}^\top \boldsymbol{y}$, and the cosine of the angle between them as $\cos(\boldsymbol{x},\boldsymbol{y})=\langle \frac{\boldsymbol{x}}{\|\boldsymbol{x}\|},\frac{\boldsymbol{y}}{\|\boldsymbol{y}\|}\rangle$. For an $n\by m$ matrix $\mathbf{A}$, we let $\|\boldsymbol{A}\|$ and $\|\boldsymbol{A}\|_F$ denote the spectral and Frobenius norm of $\boldsymbol{A}$, respectively. We also define $\one_A$ as the indicator for a statement $A$: $\one_A=1$ if $A$ is true and $\one_A=0$ otherwise. We also let $\mathcal{N}(\boldsymbol{\boldsymbol{\mu}},\boldsymbol{\Sigma}^2)$ denote the normal distribution with mean $\boldsymbol{\boldsymbol{\mu}}$ and covariance matrix $\boldsymbol{\Sigma}^2$, and $\text{Unif}(S)$ denote the uniform distribution over a set $S$. Lastly, we let $[N]$ denote the integer set $\{1,\cdots,N\}$.

\section{Problem Setting}\label{sec:settings}
We consider a binary classification problem where within each class, the data consists of subclasses with small inter-subclass correlations, formally defined as follows.

\myparagraph{Data distribution} We assume $(X,Y,Z)\sim \mathcal{D}$, where a sample $(\boldsymbol{x},y,z)\in\mathbb{R}^D\times \{-1,1\}\times \mathbb{N}_+$ drawn from $\mathcal{D}$ consists of the \emph{observed data} and \emph{class label} $(\boldsymbol{x},y)$, and a latent (unobserved) variable $z$ denoting the \emph{subclass membership}. We denote the marginal distribution of the observed part $(X,Y)$ as $\mathcal{D}_{X,Y}$. In particular, given $1\leq K_1<K$, we let
\begin{itemize}[leftmargin=9pt,topsep=0pt,parsep=0pt]
    \item $Z\sim \mathcal{D}_Z=\text{Unif}([K])$; $Y=\one_{\{Z\leq K_1\}}-\one_{\{Z>K_1\}}$;
    \item $(X\mid Z\!=\!k)\sim \mathcal{N}\lp \boldsymbol{\boldsymbol{\mu}}_k,\frac{\var^2}{D}\boldsymbol{I}\rp$, $\forall 1\leq k\leq K$, where $\var>0$ is the \emph{intra-subclass variance}, and $\{\boldsymbol{\mu}_1,\cdots,\boldsymbol{\mu}_K\}$ are \emph{subclass centers} that forms an orthonormal basis of a $K$-dimensional subspace in $\mathbb{R}^D$, where $K<D$.
\end{itemize}
This data distribution has $K$ subclasses with small inter-subclass correlation if $\alpha$ is small (since the subclass centers are orthonormal to each other), and the observed labels only reveal the superclass/class membership: the first $K_1$ subclasses belong to the positive class ($y=+1$) and the remaining $K_2:=K-K_1$ belong to the negative class ($y=-1$). One such a dataset is depicted in Figure \ref{fig_training_vis}. 

\myparagraph{Classification task and robustness of a classifier} The learning task we consider is to find a binary classifier $f(\boldsymbol{x}):\mathbb{R}^D\ra \mathbb{R}$ such that given a randomly drawn $(\boldsymbol{x},y)$, $\sign (f(\boldsymbol{x}))$ predicts the correct class label $y$, i.e., $f(\boldsymbol{x})y>0$, with high probability. Moreover, we are interested in the robustness of $f(\boldsymbol{x})$ against additive adversarial attacks/perturbations on $\boldsymbol{x}$ with $l_2$-norm bounded by some constant $r>0$ (called the \emph{attack radius}). Specifically, given a randomly drawn $(\boldsymbol{x},y)$ from $\mathcal{D}_{X,Y}$, we would like $\inf_{\|\boldsymbol{d}\|=1}f(\boldsymbol{x}+r \boldsymbol{d})y>0$ to hold with high probability (formal statement in Section \ref{sec:main_adv}) for $r$ as large as possible. 


\myparagraph{Gradient flow training} A candidate classifier can be realized from a neural network $f(\boldsymbol{x};\boldsymbol{\theta})$ parameterized by its network weights $\boldsymbol{\theta}$. Given a size-$n$ training dataset $\{(\boldsymbol{x}_i,y_i),i=1,\cdots,n\}$, where each training sample is randomly drawn from $\mathcal{D}_{X,Y}$, we define a loss function
\be
    \mathcal{L}(\boldsymbol{\theta}; \{\boldsymbol{x}_i,y_i\}_{i=1}^n)=\sum\nolimits_{i=1}^n\ell(y_i,f(\boldsymbol{x}_i;\boldsymbol{\theta}))\,,
\ee
where $\ell(y,\hat{y}):\mathbb{R}\times \mathbb{R}\ra \mathbb{R}$ is either the exponential loss $\ell(y,\hat{y})=\exp\lp -y\hat{y}\rp$ or the logistic loss $\ell(y,\hat{y})=\log\lp 1+\exp\lp -y\hat{y}\rp\rp$. We train the network using \emph{gradient flow} (GF), $\dot{\boldsymbol{\theta}}(t)\in \partial \mathcal{L}(\boldsymbol{\theta}(t))$, where $\partial \mathcal{L}$ denotes the \emph{Clarke subdifferential}~\citep{clarke1990optimization} w.r.t. $\boldsymbol{\theta}$. With proper initialization $\boldsymbol{\theta}(0)$, one expects the trained network weights $\boldsymbol{\theta}(T)$ (for some large $T>0$) to be close to some minimizer of $\mathcal{L}(\boldsymbol{\theta})$ and $f(\boldsymbol{x};\boldsymbol{\theta}(T))$ to be a good classifier for $\mathcal{D}_{X,Y}$. 

\myparagraph{Shallow networks with pReLU activation} We specifically study the following shallow (two-layer) network, parametrized by $\boldsymbol{\theta}:=\{(\boldsymbol{w}_j,v_j)\in\mathbb{R}^{D}\times \mathbb{R},j=1,\cdots,h\}$, as a candidate classifier:
\be
    f_p(\boldsymbol{x};\{\boldsymbol{w}_j,v_j\}_{j=1}^h)=\sum\nolimits_{j=1}^hv_j\frac{[\act(\lan \boldsymbol{x},\boldsymbol{w}_j\ran)]^p}{\|\boldsymbol{w}_j\|^{p-1}}\,,\tag{\emph{Shallow pReLU Network}}
\ee
where $\act(\cdot)=\mathsf{ReLU}(\cdot)=\max\{\cdot,0\}$ and $p\geq 1$. When $p=1$, this is exactly a shallow ReLU network. When $p\geq 1$, $f_p$ can be loosely viewed as a shallow network with a polynomial ReLU activation and a special form of weight normalization \cite{salimans2016weight} on each $\boldsymbol{w}_j$. One of the most important reasons for this particular generalization of the ReLU activation is the ``extra penalty" on angle separation between the input $\boldsymbol{x}$ and \emph{neurons} $\boldsymbol{w}_j$s. To see this, assume $\|\boldsymbol{x}\|=1$, then 
\begin{align*}
    \frac{[\act(\lan \boldsymbol{x},\boldsymbol{w}_j\ran)]^p}{\|\boldsymbol{w}_j\|^{p-1}}&=\;\act(\lan \boldsymbol{x},\boldsymbol{w}_j\ran)\frac{\cos^{p-1}\lp \boldsymbol{x},\boldsymbol{w}_j\rp\|\boldsymbol{w}_j\|^{p-1}}{\|\boldsymbol{w}_j\|^{p-1}}=\act(\lan \boldsymbol{x},\boldsymbol{w}_j\ran)\cos^{p-1}\lp x,\boldsymbol{w}_j\rp\,,
\end{align*}
that is, for each neuron $\boldsymbol{w}_j$, when compared to ReLU activation ($p=1$), the post-activation value is much smaller (penalized) if the angle separation between $\boldsymbol{x}$ and $\boldsymbol{w}_j$ is large. When $p>1$, such penalties, as we will see later, promote the alignment between training data samples and neurons, and result in trained networks that capture well the intrinsic structure of the data. 
In addition, such generalized ReLU activation makes the function $f_p$ \emph{positively homogeneous} of degree two w.r.t. its parameters $\boldsymbol{\theta}$, i.e., $f_p(\boldsymbol{x};\gamma\boldsymbol{\theta})=\gamma^2f_p(\boldsymbol{x};\boldsymbol{\theta})$ for any $\boldsymbol{\theta}$ and $\gamma>0$. Many existing analyses~\citep{Du&Lee, lyu2021gradient,kumar2024directional} on positively homogeneous networks apply to the generalized pReLU networks. Lastly, it is worth noting that the concept of promoting alignment between data points and neurons via cosine penalty has been explored for improving interpretability~\citep{bohle2022b} of the trained network.

\section{Main Results on Adversarial Robustness}\label{sec:main_adv}
This section considers two distinct classifiers for $\mathcal{D}_{X,Y}$:
\be
    F^{(p)}(\boldsymbol{x})=\sum_{k=1}^{K_1}\act^p(\lan \boldsymbol{\mu}_k,\boldsymbol{x}\ran)-\sum_{k=K_1+1}^{K}\act^p(\lan \boldsymbol{\mu}_k,\boldsymbol{x}\ran)\,
    \label{eq_fp}
\ee
and
\be
    F(\boldsymbol{x})=\sqrt{K_1}\act(\lan \bar{\boldsymbol{\mu}}_+,\boldsymbol{x}\ran) - \sqrt{K_2}\act(\lan \bar{\boldsymbol{\mu}}_-,\boldsymbol{x}\ran)\,,\label{eq_f0}
\ee
where $\bar{\boldsymbol{\mu}}_+\!=\!\frac{1}{\sqrt{K_1}}\sum_{k=1}^{K_1}\boldsymbol{\mu}_k$ and $\bar{\boldsymbol{\mu}}_-\!=\!\frac{1}{\sqrt{K_2}}\sum_{k=K_1+1}^{K}\boldsymbol{\mu}_k$ are, respectively, the average direction of the positive and negative subclass centers. Here the average direction is computed by taking the sum of all $\boldsymbol{\mu}_k$s, then normalize it, thus we have $\bar{\boldsymbol{\mu}}_+ \!\in\!\mathbb{S}^{D-1}$ and $\bar{\boldsymbol{\mu}}_-\!\in\!\mathbb{S}^{D-1}$.

Based on existing analyses for the implicit bias of gradient-based training, we conjecture that both classifiers can be learned by gradient flow on shallow networks with pReLU activations and proper initialization. Before studying this conjecture, we analyze these two classifiers in more detail.

\myparagraph{Realizability of $F(\boldsymbol{x})$ and $F^{(p)}(\boldsymbol{x})$ by pReLU networks} We first notice that both $F(\boldsymbol{x})$ and $F^{(p)}(\boldsymbol{x})$ can be realized (non-uniquely) by a pReLU network $f_p(\boldsymbol{x};\{\boldsymbol{w}_j,v_j\}_{j=1}^h)$ for suitable choices of the weights $\{\boldsymbol{w}_j,v_j\}_{j=1}^h$, as stated in the following claim. (Verifying this claim is straightforward and we leave it to Appendix \ref{app_pf_thm}).

\medskip
\begin{claim}
    The following two statements are true:
    \begin{itemize}[leftmargin=9pt,topsep=0pt,parsep=0pt]
       \item 
        When $p=1$, let $I_+$ and $I_-$ be any two non-empty and disjoint subsets of $[h]$, let $\boldsymbol{\lambda}:=(\lambda_j) \in \mathbb{R}^h_{\geq 0}$ be any vector such that $\sum_{j\in I_+}\lambda_j^2=\sqrt{K_1}$ and $\sum_{j\in I_-}\lambda_j^2=\sqrt{K_2}$, and let
        \begin{align*}
        & \boldsymbol{w}_j=\lambda_j\bar{\boldsymbol{\mu}}_+,\ v_j=\lambda_j, &\forall j\in I_+\,,\\
        & \boldsymbol{w}_j=\lambda_j\bar{\boldsymbol{\mu}}_-,\ v_j=-\lambda_j, &\forall j\in I_-\,,\\
        & \boldsymbol{w}_j = 0,\  v_j = 0, &\text{otherwise}. 
        \end{align*}
        Then $f_p(x;\{\boldsymbol{w}_j,v_j\}_{j=1}^h)\equiv F(\boldsymbol{x})\,.$
        \item When $p\geq 1$, let $I_1,\cdots, I_K$ be any $K$ non-empty and disjoint subsets of $[h]$, let $\boldsymbol{\lambda}:=(\lambda_j) \in \mathbb{R}^h_{\geq 0}$ be any vector such that $\sum_{j\in I_k}\lambda_j^2=1,\forall k\in[K]$, and let 
        \begin{align*}
        & \boldsymbol{w}_j=\lambda_j\boldsymbol{\mu}_k,\ v_j=\lambda_j, &\forall j\in I_k, k\leq K_1\,,\\
        & \boldsymbol{w}_j=\lambda_j\boldsymbol{\mu}_k,\ v_j=-\lambda_j, &\forall j\in I_k, k>K_1\,,\\
        & \boldsymbol{w}_j = 0,\ v_j = 0 & \text{otherwise}. 
        \end{align*}
        Then $f_p(x;\{\boldsymbol{w}_j,v_j\}_{j=1}^h)\equiv F^{(p)}(\boldsymbol{x})\,.$
    \end{itemize}
\end{claim}
\begin{remark}
    As shown in the claim, there are infinitely many parameterizations of $f_p$ (determined by the choice of subsets of $[h]$ and $\boldsymbol{\lambda}$) that lead to the same function $F(\boldsymbol{x})$ (or $F^{(p)}(\boldsymbol{x})$), due to the \emph{symmetry} in the network weights. That is, the resulting network represents the same function if one re-indexes (permutes) the neurons, or if one does the following rescaling on some weights $(\boldsymbol{w}_j,v_j)\rightarrow(\gamma \boldsymbol{w}_j,v_j/\gamma)$ for some $j\in[j]$ and some $\gamma>0$.
\end{remark}
Notice that the classifier $F$ corresponds to a vanilla shallow ReLU network whose non-zero neurons are either aligned with $\bar{\boldsymbol{\mu}}_+$, the average direction of the positive subclass centers, or $\bar{\boldsymbol{\mu}}_-$, the negative counterpart. On the other hand, $F^{(p)}$ corresponds to a pReLU network whose non-zero neurons are aligned with one of the subclass centers. As we will see soon, both $F(\boldsymbol{x})$ and $F^{(p)}(\boldsymbol{x})$ achieve high prediction accuracy on clean samples $(\boldsymbol{x},y)$ from $D_{X,Y}$, i.e., learning subclass centers is not necessary for good generalization. However, we will also show that $F^{(p)}$ is much more robust against adversarial perturbation than $F$. This is why we argue that \emph{learning subclass centers significantly improve robustness}.

\myparagraph{Generalization and Robustness of $F(\boldsymbol{x})$ and $F^{(p)}(\boldsymbol{x})$} Having established that both $F$ and $F^{(p)}$ are realizable by a pReLU network $f_p$, we now study the generalization performance and robustness of $F$ and $F^{(p)}$. Under the conjecture that both $F$ and $F^{(p)}$ can be learned by gradient flow training on a pReLU network $f_p$, we expect that such generalization and robustness properties extend to $f_p$.

The following result states that both $F$ and $F^{(p)}$ achieve good generalization performance on $\mathcal{D}_{X,Y}$ when the dimension $D$ of the input data $X$ is sufficiently large and the number of subclasses $K$ is not too large.
\begin{proposition}[Generalization on clean data]\label{prop_gen}
    Given a test sample $(\boldsymbol{x},y)\sim \mathcal{D}_{X,Y}$, and classifiers $F$ and $F^{(p)}$, $p\geq 1$, defined in \eqref{eq_f0} and \eqref{eq_fp}, resp., there is a constant $C$ such that
    \begin{align*}
        \prob\lp F(\boldsymbol{x})y>0\rp &\geq 1-4\exp\lp -\frac{CD}{4\var^2K}\rp,\\
        \prob\lp F^{(p)}(\boldsymbol{x})y>0\rp &\geq 1-2(K+1)\exp\lp -\frac{CD}{\var^2K^2}\rp.
    \end{align*}
\end{proposition}

However, the following theorem shows a significant difference between $F$ and $F^{(p)}$ regarding their adversarial robustness, when the number of subclasses $K$ is large.
\begin{theorem}[$l_2$-adversarial robustness]\label{thm_robust}
  Given a test sample $(\boldsymbol{x},y)\sim \mathcal{D}_{X,Y}$, and classifiers $F$ and $F^{(p)}$, defined in \eqref{eq_f0} and \eqref{eq_fp}, resp.,
  the following statement is true for the same constant $C$ in Proposition \ref{prop_gen}:
    \begin{itemize}[leftmargin=9pt,topsep=0pt,parsep=0pt]
        \item $\lp \text{Non-robustness of } F \text{ against }\mathcal{O}\lp\frac{1}{\sqrt{K}}\rp \text{-attack} \rp$ Let $d_0:=\frac{\sqrt{K_1}\bar{\boldsymbol{\mu}}_+-\sqrt{K_2}\bar{\boldsymbol{\mu}}_-}{\sqrt{K}}\in\mathbb{S}^{D-1}$. Then for any $\rho\geq 0$,
        \begin{align*}
            &\;\prob\!\lp F\!\lp \boldsymbol{x}-\frac{y(1+\rho)}{\sqrt{K}}\boldsymbol{d}_0\rp y\!>\!0\rp \!\leq\! 2\exp\!\lp -\frac{CD\rho^2}{K\var^2}\rp\!.
        \end{align*}
        \item $\lp \text{Robustness of } F^{(p)} \text{ against }\mathcal{O}\lp 1\rp \text{-attack }\rp$ Let $p\geq 2$. Then for any $0\leq \delta\leq \sqrt{2}$,
        \begin{align*}
            &\;\prob\lp \min_{\|\boldsymbol{d}\|\leq 1} F^{(p)}\lp \boldsymbol{x}+\frac{\sqrt{2}-\delta}{2}\boldsymbol{d}\rp y>0\rp \geq 1-2(K+1)\exp\lp -\frac{CD\delta^2}{2K^2\var^2}\rp\,.
        \end{align*}
    \end{itemize}
\end{theorem}
We refer the reader to Appendix \ref{app_pf_thm} for the proofs for Proposition \ref{prop_gen} and Theorem \ref{thm_robust}. Our theoretical results should be understood in the high-dimensional or low-intra-class-variance regime so that $\frac{D}{\alpha^2}\gg K^2\log K$. On the one hand, Theorem \ref{thm_robust} shows that $F$, a classifier that corresponds to shallow ReLU networks whose neurons are aligned with either $\bar{\boldsymbol{\mu}}_+$ or $\bar{\boldsymbol{\mu}}_-$, is vulnerable to an adversarial attack of radius $\mathcal{O}\lp \frac{1}{\sqrt{K}}\rp$, which diminishes as the number of subclasses grows. On the other hand, Theorem \ref{thm_robust} shows that $F^{(p)}$, $p\geq 2$, a classifier that corresponds to shallow pReLU networks whose neurons are aligned with one of the subclass centers, is $\mathcal{O}\lp 1\rp$-robust against adversarial attacks. 

\smallskip
\begin{remark}
    Complementary results to Theorem \ref{thm_robust} in Appendix \ref{app_robust_comp} show that $F$ is robust against any attack with $l_2$-norm slightly smaller than $\frac{1}{\sqrt{K}}$, and $F^{(p)}$ is not robust to some attack with $l_2$-norm slightly larger than $\frac{\sqrt{2}}{2}$. Therefore, $\frac{1}{\sqrt{K}}$ and $\frac{\sqrt{2}}{2}$ are the ``critical" attack radius for $F$ and $F^{(p)}$, resp. We conjecture this is also the case for the associated trained networks, which we verify in Section \ref{ssec:val}.
\end{remark}

\myparagraph{Conjecture on the outcome of gradient flow training} 
We now study the conjecture that both classifiers can be learned by gradient flow on a shallow network with pReLU activations and proper initialization. Specifically, we pose the following:

\begin{conjecture}\label{conj_conv}
    Suppose that the intra-subclass variance $\alpha>0$ is sufficiently small, that one has a training dataset of sufficiently large size, and that we run gradient flow training on $f_p(\boldsymbol{x};\boldsymbol{\theta}), \boldsymbol{\theta}=\{\boldsymbol{w}_j,v_j\}_{j=1}^h$ of sufficiently large width $h$ for sufficiently long time $T$, starting from random initialization of the weights with a sufficiently small initialization scale. If $p=1$, then we have $\inf_{c>0}\sup_{\boldsymbol{x}\in\mathbb{S}^{D-1}}|cf_p(\boldsymbol{x};\boldsymbol{\theta}(T))-F(\boldsymbol{x})|\ll 1$; If $p\in[3,\bar{p})$ for some $\bar{p}>3$, then we we have $\inf_{c>0}\sup_{\boldsymbol{x}\in\mathbb{S}^{D-1}}|cf_p(\boldsymbol{x};\boldsymbol{\theta}(T))-F^{(p)}(\boldsymbol{x})|\ll 1$.
\end{conjecture}

Our conjecture states that with proper initialization and sufficiently long training time, gradient flow finds a network that is, up to a scaling factor $c>0$\footnote{Since $f_p(x;\boldsymbol{\theta})$ is positively homogeneous of degree two w.r.t. its parameter $\boldsymbol{\theta}$, $\|\boldsymbol{\theta}(t)\|\ra \infty$ as training time $t\ra \infty$~\citep{lyu2019gradient,ji2020directional}, the output of $f_p(\boldsymbol{x};\boldsymbol{\theta}(T))$ can never match that of $F(\boldsymbol{x})$ or $F^{(p)}(\boldsymbol{x})$ without a proper choice of scaling factor $c>0$. However, we note, that such a scaling factor will not change the prediction $\sign(f_p(\boldsymbol{x};\boldsymbol{\theta}(T)))$.}, close to $F$ in $l_\infty$-distance over $\mathbb{S}^{D-1}$, if $p=1$. That is, when training a vanilla ReLU network, the neurons learn average directions $\bar{\boldsymbol{\mu}}_+$, and $\bar{\boldsymbol{\mu}}_-$ of subclass centers. However, when $p\geq 3$, gradient flow finds a network close to $F^{(p)}$, i.e., the neurons learn individual subclass centers. Note that $p$ can not be too large, as the post activation $\frac{[\act(\lan \boldsymbol{x},\boldsymbol{w}_j\ran)]^p}{\|\boldsymbol{w}_j\|^{p-1}}$ converges to $\one_{\cos\lp \boldsymbol{x},\boldsymbol{w}_j\rp=1}\cdot\lan \boldsymbol{x},\boldsymbol{w}_j\ran$ when $p$ grows ($\|\boldsymbol{x}\|=1$), effectively zeroing out post activation values almost everywhere and also the gradient, staggering gradient flow training. 

Given Conjecture \ref{conj_conv}, Theorem \ref{thm_robust} can be used to infer the robustness of the networks $f_p(\boldsymbol{x};\boldsymbol{\theta}(T))$ obtained from gradient flow training with small initialization. When training a vanilla ReLU network in this regime, we expect the trained network to be close to $F$, thus non-robust against $\mathcal{O}\lp \frac{1}{\sqrt{K}}\rp$-attacks. When training a pReLU network with $p\geq 3$, we expect the trained network to be close to $F^{(p)}$, which is more robust than its ReLU counterpart. 

Our conjecture is based on existing work on the implicit bias of gradient flow on shallow networks with small initialization, and we carefully explain the rationale behind it in Section \ref{sec:theoretical_analsysis}. However, formally showing such convergence results is challenging as the data distribution $\mathcal{D}_{X,Y}$ considered here is more complicated than those for which convergence results can be successfully shown~\cite{boursier2022gradient,min2023early,chistikov2023learning, wang2023understanding}. Instead, we provide a preliminary analysis of this conjecture. Moreover, in Section \ref{sec:num}, we provide numerical evidence that supports our conjecture.

\myparagraph{Comparison with prior work}
\citet{frei2023the} consider a similar setting as ours with three minor differences. Firstly, their data distribution differs from ours by a scaling factor of $\sqrt{D}$ on the input data. Second, they allow for tiny correlations between the two subclass centers. Third, they consider shallow ReLU networks with bias.
They show that any network trained by gradient flow is non-robust against $\mathcal{O}\lp \frac{1}{\sqrt{K}}\rp$ attacks, which covers any initialization, but their analysis does not characterize what classifier the trained network corresponds to. While we consider specifically small initialization, we can at least conjecture, and numerically verify, what the corresponding classifier is at the end of the training. In addition, \citet{frei2023the} show the existence of $\mathcal{O}(1)$-robust ReLU network if neurons are aligned with subclasses and there is some carefully chosen bias. While it already sheds light on the need for learning subclasses, their proposed network can not be found by gradient flow. On the contrary, our proposed robust $F^{(p)}$ can be obtained by gradient flow.


\section{Discussion on Gradient Flow Training}\label{sec:theoretical_analsysis}
In Section \ref{sec:main_adv}, we conjectured that under gradient flow starting from a small initialization, a vanilla ReLU network favors learning average directions $\bar{\boldsymbol{\mu}}_+$ and $\bar{\boldsymbol{\mu}}_-$ of subclass centers, while a pReLU network favors learning every subclass centers $\boldsymbol{\mu}_k,\ k\in[K]$, resulting in a significant difference between these two networks in terms of adversarial robustness.
In this section, we provide a theoretical explanation of such alignment preferences in the scope of implicit bias of gradient flow training with small initialization~\citep{maennel2018gradient,boursier2024early}. We remind the reader that gradient flow training is described in Section \ref{sec:settings}.


\subsection{Gradient flow with small initialization} 
We first briefly describe the training trajectory of gradient flow on shallow networks with a small initialization. With a sufficiently small initialization scale (to be defined later), the gradient flow training is split into two phases with distinct dynamic behaviors of the neurons. The first phase is often referred to as the initial/early \emph{alignment phase}~\citep{boursier2024early,min2023early}, during which the neurons keep small norms while changing their directions towards one of the \emph{extremal vectors}~\citep{maennel2018gradient}, which are jointly determined by the training dataset and the activation function. The second phase is often referred to as the fitting/convergence phase. Within the fitting phase, neurons keep a good alignment with the extremal vectors and start to grow their norms, and the loss keeps decreasing until convergence, upon which one obtains a trained network whose dominant neurons (those with large norms) are all aligned with one of the extremal vectors. 

Notably, the neurons favor different extremal vectors depending on the activation function, precisely depicted by Figure \ref{fig_training_vis}. With the same dataset and with the same initialization of the weights, pReLU activation makes a significant difference in neuron alignment: When $p=1$ (vanilla ReLU network), the average class centers $\bar{\boldsymbol{\mu}}_+$ and $\bar{\boldsymbol{\mu}}_-$ are those extremal vectors ``attracting" neurons during the alignment phase, leading to a trained ReLU network that has effectively two neurons (one aligned with $\bar{\boldsymbol{\mu}}_+$ and another with $\bar{\boldsymbol{\mu}}_-$) at the end of the training. However, when $p=3$, the subclass centers $\boldsymbol{\mu}_1,\cdots,\boldsymbol{\mu}_k$ become extremal vectors that are ``attracting" neurons
, the resulting pReLU networks successfully learn every subclass center, which, we have argued in Section \ref{sec:main_adv}, substantially improves the robustness (over vanilla ReLU net) against adversarial attack. While the case visualized in Figure \ref{fig_training_vis} is a simple one in 3-dimension with 3 subclasses, we will provide experiments in higher dimension and with more subclasses in Section \ref{sec:num}.

Despite the seemingly simple dynamic behavior of the neurons, the rigorous theoretical analysis is much more challenging due to the complicated dependence~\citep{maennel2018gradient,boursier2022gradient,wang2023understanding, kumar2024directional} of the extremal vectors on the dataset and the activation function. Hence prior works \citep{boursier2022gradient,min2023early, chistikov2023learning, wang2023understanding} assume simple training datasets and their analyses are for ReLU activation. Here, we are faced with a relatively more complex data distribution that generates our dataset, and at the same time deals with a pReLU activation, thus a full theoretical analysis of the convergence, that would prove our Conjecture \ref{conj_conv}, still has many challenges and deserves a careful treatment in a separate future work. For now, we provide some preliminary theoretical analysis of the neural alignment in pReLU networks that supports our conjecture.

\subsection{Preliminary theoretical analysis on neuron alignment in pReLU networks}
We first make the following two simplifying assumptions:

\myparagraph{Small and balanced initialization} First, we obtain i.i.d. samples $w_{j0}, j=1,\cdots,h$ drawn from some random distribution such that almost surely $\|w_{j0}\|\leq M,\forall j$ for some $M>0$\footnote{For example, if we sample every entries of $w_{j0}$ by a uniform distribution over $\lp -\frac{1}{\sqrt{D}}, -\frac{1}{\sqrt{D}}\rp$.}, and then and then initialize the weights as
\be
    \boldsymbol{w}_j(0)=\epsilon w_{j0},\ \ v_j(0)\in\{\|\boldsymbol{w}_j(0)\|,-\|\boldsymbol{w}_j(0)\|\}, \forall j\in[h]\,.\label{eq_init}
\ee
That is, $w_{j0}$ determines the initial direction of the neurons and we use $\epsilon$ to control the initialization scale. This balanced assumption is standard in the analysis of shallow networks with small initialization \citep{soltanolkotabi2023implicit,boursier2022gradient,min2023early}.

\myparagraph{Simplified training dataset} The training dataset $\{(\boldsymbol{x}_k,y_k), k=1,\cdots,K\}$ is the collection of exact subclass centers, together with their class label. That is, $\boldsymbol{x}_k=\boldsymbol{\mu}_k,\quad y_k=\one_{\{k\leq K_1\}}-\one_{\{k>K_1\}},\forall k\in[K]$. Similar datasets with orthogonality among data points $\boldsymbol{x}_k$s are studied in~\citet{boursier2022gradient} for ReLU networks.

\myparagraph{Neuron alignment in pReLU networks} The key to the theoretical understanding of neuron alignment is the following lemma.
\begin{lemma}\label{lem_small_norm_informal}
Given some initialization from \eqref{eq_init}, if $\epsilon=\mathcal{O}( \frac{1}{\sqrt{h}})$, then there exists $T=\boldsymbol{\theta}(\frac{1}{K}\log\frac{1}{\sqrt{h}\epsilon})$ such that the trajectory under gradient flow training with the simplified training dataset almost surely satisfies that $\forall t\leq T$,
\begin{align*}
    &\;\max_j\lV \frac{d}{dt} \frac{\boldsymbol{w}_j(t)}{\|\boldsymbol{w}_j(t)\|}-\sign(v_j(0)) \mathcal{P}^\perp_{\boldsymbol{w}_j(t)} \boldsymbol{x}^{(p)}(\boldsymbol{w}_j(t))\rV=\mathcal{O}\lp\epsilon k\sqrt{h}\rp\,,\label{eq_err}
\end{align*}
where $\mathcal{P}^\perp_{\boldsymbol{w}}=I-\frac{\boldsymbol{w}\boldsymbol{w}^\top}{\|\boldsymbol{w}\|^2}$ and
\be
    \boldsymbol{x}^{(p)}(w)=\sum\nolimits_{k=1}^K\gamma_k(w)y_k\boldsymbol{x}_k\cdot p[\cos(\boldsymbol{x}_k,w)]^{p-1},
\ee
with $\gamma_k(\boldsymbol{w})$ being a subgradient of $\act(z)$ at $z=\lan \boldsymbol{x}_k,\boldsymbol{w}\ran$ when $p=1$ and $\gamma_k(\boldsymbol{w})=\one_{\cos(\boldsymbol{x}_k,w)\geq 0}$ when $p>1$.
\end{lemma}
Lemma \ref{lem_small_norm_informal} suggests that during the alignment phase $t\leq T$, one have the following approximation
\be
    \frac{d}{dt} \frac{\boldsymbol{w}_j(t)}{\|\boldsymbol{w}_j(t)\|}\simeq \sign(v_j(0)) \mathcal{P}^\perp_{\boldsymbol{w}_j(t)} x^{(p)}(\boldsymbol{w}_j(t))\,,
\ee
which essentially shows that when $\boldsymbol{w}_j$ is a \emph{positive neuron} ($\sign(v_j(0))>0$), then gradient flow dynamics during alignment phase pushes $\boldsymbol{w}_j/\|\boldsymbol{w}_j\|$ toward the direction of $\boldsymbol{x}^{(p)}(\boldsymbol{w}_j)$. Notably, $\boldsymbol{x}^{(p)}(\boldsymbol{w}_j)$ is a weighted combination of $\boldsymbol{x}_k$s and the mixing weights critically depend on $p$. Roughly speaking, when $p=1$, $\boldsymbol{x}^{(p)}(\boldsymbol{w}_j)$, weighing equally on $\boldsymbol{x}_k$s that activate $\boldsymbol{w}_j$, are more aligned with $\bar{\boldsymbol{\mu}}_+$ and $\bar{\boldsymbol{\mu}}_-$, while when $p\geq 3$, $\boldsymbol{x}^{(p)}(\boldsymbol{w}_j)$, weighing more on $\boldsymbol{x}_k$s that are close to $\boldsymbol{w}_j$ in angle, are more aligned with one of the subclass centers, thus by moving toward $\boldsymbol{x}^{(p)}(\boldsymbol{w}_j)$ in direction, the neuron $\boldsymbol{w}_j$ is likely to align with average class centers in the former case, and with subclass centers in the latter case (We illustrate this in an example in Appendix \ref{app_align_bias_vis}). This trend leads to the following alignment bias.
\begin{theorem}[Alignment bias of positive neurons]\label{thm_align}
    Given some sufficiently small $\delta>0$ and a fixed choice of $p\geq 1$, then $\exists \epsilon_0:=\epsilon_0(\delta,p)>0$ such that for any solution of the gradient flow on $f_p(\boldsymbol{x};\boldsymbol{\theta})$ with the simplified training dataset, starting from some initialization from \eqref{eq_init} with initialization scale $\epsilon<\epsilon_0$, almost surely we have that at any time $t\leq T=\boldsymbol{\theta}(\frac{1}{n}\log\frac{1}{\sqrt{h}\epsilon})$ and $\forall j$ with $\sign(v_j(0))>0$, 
    \be
        \left.\frac{d}{dt}\! \cos\!\lp \boldsymbol{w}_j(t), \bar{\boldsymbol{\mu}}_+\rp \rvt_{\cos(\boldsymbol{w}_j(t),\bar{\boldsymbol{\mu}}_+)=1-\delta}\! \begin{cases}
            >0,&\! \text{when } p=1\\
            <0,&\! \text{when } p\geq 3\\
        \end{cases},
    \ee
    and
    \be
        \left.\frac{d}{dt}\! \cos\!\lp \boldsymbol{w}_j(t), \boldsymbol{\mu}_k\rp \rvt_{\cos(\boldsymbol{w}_j(t),\boldsymbol{\mu}_k)=1-\delta}>0,\forall k\leq K_1\,.\label{eq_align_pref_subclass}
    \ee
\end{theorem}
This theorem shows how different choice of $p$ alters the neurons' preferences on which directions to align. If we define the neighborhood of some $\boldsymbol{\mu}$ direction as $\mathcal{B}(\boldsymbol{\mu}, \delta):=\{\boldsymbol{z}\in\mathbb{S}^{D-1}:\cos\lp \boldsymbol{\mu},\boldsymbol{z}\rp\geq 1-\delta\}$, then specifically for a positive neuron: When $p=1$, then any neuron with $\boldsymbol{w}_j(t_0)\in \mathcal{B}(\bar{\boldsymbol{\mu}}_+, \delta)$ necessarily has $\boldsymbol{w}_j(t)\in \mathcal{B}(\bar{\boldsymbol{\mu}}_+, \delta),\forall t\in[t_0,T]$, i.e. it keeps at least $1-\delta$ alignment with $\bar{\boldsymbol{\mu}}_+$ until the end of the alignment phase (at time $T$). Therefore, there is a preference of staying close to $\bar{\boldsymbol{\mu}}_+$ for positive neurons. Interestingly, such preference no longer exists when $p\geq 3$. In particular, any positive neuron with $\boldsymbol{w}_j(t_0)\notin \mathcal{B}(\bar{\boldsymbol{\mu}}_+, \delta)$ necessarily has $\boldsymbol{w}_j(t)\notin \mathcal{B}(\bar{\boldsymbol{\mu}}_+, \delta),\forall t\in[t_0,T]$, i.e., any neuron initialized with some angular distance to $\bar{\boldsymbol{\mu}}_+$ will not get any closer to $\bar{\boldsymbol{\mu}}_+$. Instead, the neurons are now more likely to stay close to some subclass centers, as shown in \eqref{eq_align_pref_subclass}. Similar results can be said for negative neurons (whose index $j$ has $\sign(v_j(0))<0$): they favor average negative class center $\bar{\boldsymbol{\mu}}_-$ when $p=1$ and favor subclass directions when $p\geq 3$. We refer the interested readers to Appendix \ref{app_pf_align} for a complete Theorem \ref{thm_align} (and its proof) that also includes the statement for negative neurons.


\section{Numerical Experiments}\label{sec:num}
Our numerical experiments\footnote{Code available at \url{https://github.com/hanchmin/pReLU_ICML24}.} have two parts: First, we conduct experiments to validate Conjecture \ref{conj_conv}. Then, we provide preliminary experiments on real datasets to highlight the potential of using pReLU activation for obtaining more robust classifiers.
\subsection{Numerical evidence supports our conjecture}\label{ssec:val}
We run SGD (batch size $100$ and step size $0.2$ for $2\times 10^5$ epochs) with small initialization (all weights initialized as mean-zero Gaussian with standard deviation $10^{-7}$) to train a pReLU network with $h=2000$ neurons on a dataset drawn from $\mathcal{D}_{X,Y}$ with $D=1000$, $K=10$, and $K_1=6$. With different choices of intra-subclass variance $\alpha$, we take the network $f_p$ at the end of the training and estimate (we refer the readers to Appendix \ref{app_add_conj} for the estimation algorithm) its distance $\text{dist}(f_p,F)=\inf_{c>0}\sup_{\boldsymbol{x}\in\mathbb{S}^{D-1}}|cf_p(\boldsymbol{x})-F(\boldsymbol{x})|$ to the classifier $F(\boldsymbol{x})$, or the distance $\text{dist}(f_p,F^{(p)})$ to $F^{(p)}(\boldsymbol{x})$, depending on the value of $p$. As one sees from Figure \ref{fig_conjecture_val}, when $p=1$, $f_p$ is close to $F$, and the estimated distance slightly increases as $\alpha$ gets larger. Similarly, when $p=3$, $f_p$ is close to $F^{(p)}$. Furthermore, we investigate the alignment between the dominant neurons in $f_p$ and class/subclass centers. Figure \ref{fig_conjecture_val} shows that indeed neurons in $f_p$ learn only average class centers $\bar{\boldsymbol{\mu}}_+$ and $\bar{\boldsymbol{\mu}}_-$ when $p=1$ while learning every subclass center $\boldsymbol{\mu}_k,k\in[K]$ when $p=3$. Lastly, as our Theorem \ref{thm_robust} predicts, $f_p, p=3$ is much more robust than the one with $p=1$. This series of numerical evidence strongly support Conjecture \ref{conj_conv} (with additional experiments in Appendix \ref{app_add_conj}).

\begin{figure*}[t]
  \centering
  \vspace{-0.3cm}
  \includegraphics[width=\linewidth]{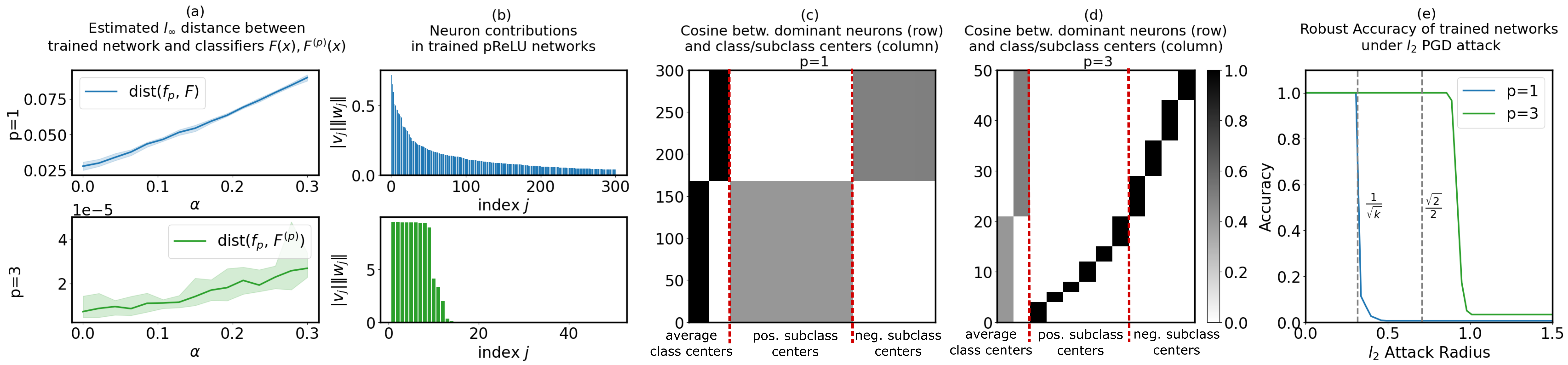}
  \vspace{-0.7cm}
  \caption{Numerical experiment ($K=10,K_1=6$) validates Conjecture \ref{conj_conv}. \textbf{(a)} We train pReLU networks using SGD with small initialization, then estimate the distance $\text{dist}(f_p,F)$ between the trained network $f_p$ and the classifier $F$, when $p=1$ (top plot); When $p=3$, we estimate $\text{dist}(f_p,F^{(p)})$ instead (bottom plot). The training is done under different choices of intra-subclass variance $\alpha$ and repeated 10 times per $\alpha$; the Solid line shows the average and the shade denotes the region between max and min values. \textbf{(b)} Given a trained network obtained from an instance of this training ($\alpha=0.1$), we reorder the neurons w.r.t. their contributions $|v_j|\|\boldsymbol{w}_j\|$ and then plot the contributions in a bar plot; \textbf{(c)(d)} For neurons with large contributions, we plot a colormap, with each pixel represents some  $\cos(\boldsymbol{w}_j,\boldsymbol{\mu})$, where $\boldsymbol{\mu}$ could be either average class centers $\bar{\boldsymbol{\mu}}_+$ and $\bar{\boldsymbol{\mu}}_-$ or subclass centers $\boldsymbol{\mu}_k,k\in[K]$. Note: For visibility, the neurons are reordered again so that neurons aligned with the same $\boldsymbol{\mu}$ are grouped together. \textbf{(e)} Lastly, we carry out $l_2$ PGD attack on a test dataset and plot the robust accuracy of the trained network under different choices of attack radius.}
  \label{fig_conjecture_val}
  \vspace{-0.4cm}
\end{figure*}
\subsection{Experiments on real datasets}
Although we have shown good alignment between our theory and our numerical experiments in \ref{ssec:val}. The settings largely remain unrealistic. To show the potential value of pReLU networks in practical settings in finding a robust classifier, we now study classification (albeit slightly modified) problems on real datasets. 
\subsubsection{Parity Classification on MNIST}\label{ssec:mnist}
We first consider training a pReLU network of width $h=500$ to predict whether an MNIST digit is even or odd. This poses a problem similar to the one studied in our theoretical analyses: each digit is naturally a subclass and they form classes of even digits and odd digits. Moreover, once the data is centered, as shown in Figure \ref{fig_mnist_training}, two data points of the same digit are likely to have a large positive correlation, and two points of different digits to have a small, or negative correlation, which resembles the our orthogonality assumption among subclass centers. Therefore, with appropriate data preprocessing, we expect the pReLU network to find a more robust classifier when $p$ is large.

\myparagraph{Data preprocessing} We relabel MNIST data by parity, which leads to a binary classification. Then we center both the training and test set by subtracting off the mean image of all training data and then normalize the residue, resulting in a centered, normalized training/test set \footnote{When training pReLU networks, some normalization of the data is required to improve training stability. To see this, notice that the post-activation value for $i$th data scales as $\|x_i\|^p$; When $p$ is large, this term diminishes or explodes depending on where the value $\|x\|$ is smaller or larger than one.}. 

\myparagraph{Training pReLU} We use Kaiming initialization~\cite{he2015delving} with non-small variance for all the weights and run Adam with cross-entropy loss and with batch size $1000$ for 50 epochs and summarize the training results in Figure \ref{fig_mnist_training}. First of all, as $p$ increases, the trained network becomes more robust against the adversarial $l_\infty$-attack computed from an \emph{adaptive projected gradient ascent} (APGD) algorithm~\cite{croce2020reliable}. Interestingly, pReLU with $p>1$ even has a slight edge over vanilla ReLU net in terms of test accuracy on clean data. Given that the MNIST dataset does not satisfy our data distribution, there is no reason to expect that neurons in pReLU can learn each subclass (in this case, the individual digit) and indeed they cannot. However, we highlight that pReLU empirically is more capable of learning distinct patterns within each superclass/class: We stack the hidden post-activation representation of each training sample into a matrix and compute its stable rank (defined as $\frac{\|\boldsymbol{A}\|_F^2}{\|\boldsymbol{A}\|^2}$ for a matrix $\boldsymbol{A}$, as an approximation of $\rank(\boldsymbol{A})$). From Figure \ref{fig_mnist_training}, we see that the hidden feature matrix of MNIST obtained from pReLU network has a much larger stable rank than the one from vanilla ReLU net, i.e. the hidden features collapse less when $p$ is large, and we conjecture it to be the reason why pReLU still gains much more adversarial robustness than vanilla ReLU. Theoretically investigating such a phenomenon is an interesting future direction.

\begin{figure*}[t]
  \centering
  \vspace{-0.2cm}
  \includegraphics[width=\linewidth]{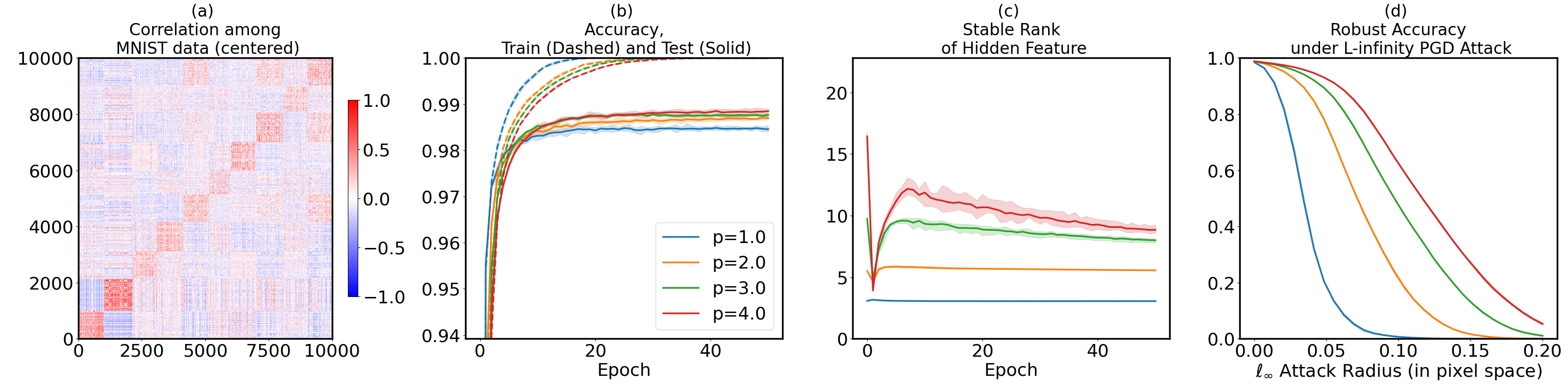}
  \vspace{-0.9cm}
  \caption{Parity prediction on MNIST dataset with pReLU networks. \textbf{(a)} We plot the data correlation as a colormap, where each pixel represents some $\cos\lp x_i,x_j\rp$ between two centered data $x_i,x_j$ from MNIST training dataset; \textbf{(b)} We run Adam with batch size $1000$ to train a pReLU network under Kaiming initialization (repeated 10 times), then plot the training/testing accuracy during training for different choice of $p$ (The shade region indicates the range between the minimum and maximum values over 10 randomized runs); \textbf{(c)} We stack the hidden post-activation representation of each training sample into a matrix and compute its stable rank, and plot the evolution of this stable rank during training; \textbf{(d)} After training for $50$ epoch, we carry out APGD $l_\infty$-attack on MNIST test dataset (in pixel space) and plot the robust accuracy of the trained pReLU network under different choice of attack radius.} 
  \label{fig_mnist_training}
\end{figure*}
\begin{figure*}[t]
  \centering
  \vspace{-0.3cm}
  \includegraphics[width=\linewidth]{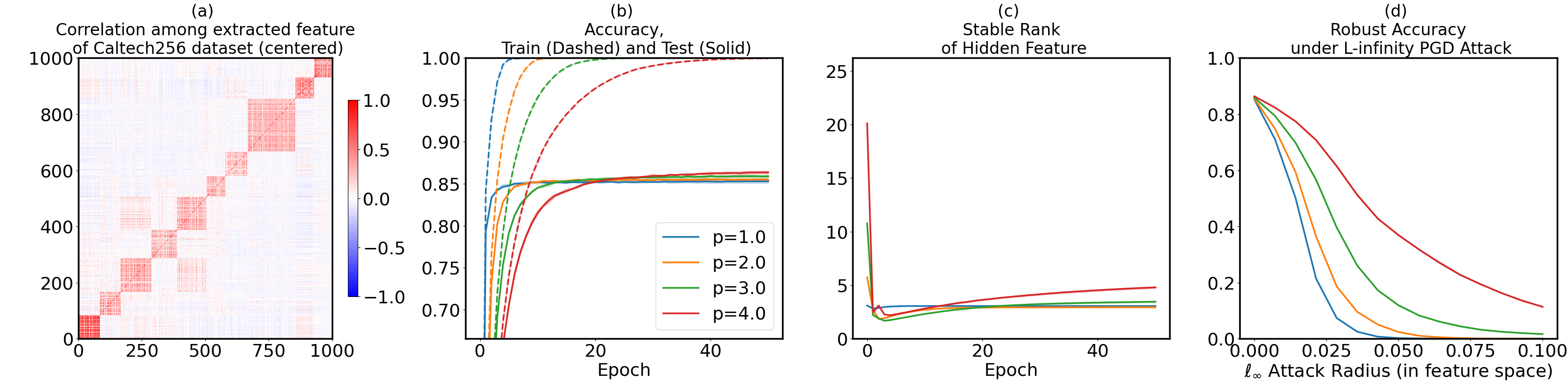}
  \vspace{-0.7cm}
  \caption{Classification on Caltech256 dataset (relabeled into 10 superclasses) with a pre-trained ResNet152 as a fixed feature extractor.}
  \label{fig_caltech256_training}
  \vspace{-0.4cm}
\end{figure*}

\myparagraph{Additional experiments} To further illustrate the effectiveness of pReLU activation in improving the adversarial robustness of the trained network, we conduct additional experiments, in Appendix \ref{app_mnist}, of training pReLU networks for the original digit classification task on MNIST and test the robustness of the trained network with different types of attacks. Despite that our theoretical results only suggest robustness gain can be achieved under datasets with subclasses, the additional experimental results show that the pReLU networks are still much more robust than their ReLU counterpart even for the original digit classification task, which clearly does not follow our data assumption.

\subsubsection{Image classification with pre-trained feature extractor}
Our next experiment considers a transfer learning setting where we use some intermediate layer (more precisely, the feature layer before the fully-connected layers) of a pre-trained ResNet152 on ImageNet as a feature extractor and classify the extracted feature from Caltech256 \cite{griffin2007caltech} object classification dataset with pReLU networks. The main reason behind this setting is that we expect the extracted feature representation of each class to be clustered, and this is verified in Figure \ref{fig_caltech256_training}. 

We intend to have a classification task with subclasses, thus we group the original 256 classes in the dataset into 10 superclasses (each superclass has no semantic meaning in this case) and train pReLU networks of width $h=2000$ that predict the superclass label. Then we obtain the feature representation of the dataset from our feature extractor, center the feature, and normalize, same as we did for the MNIST dataset. The training process is exactly the same as for the MNIST dataset, and we summarize the results in Figure \ref{fig_caltech256_training}. Even for this multi-class classification task, still pReLU achieves higher test accuracy and is more robust when $p$ gets larger. 
\section{Conclusion}
By introducing a generalized pReLU activation, we resolve the non-robustness issue, caused by the implicit bias of gradient flow, of ReLU networks when trained on a classification task in the presence of latent subclasses with small inter-subclass correlations. Future work includes formal analyses on neural alignment in pReLU networks as well as more empirical evaluation of pReLU activation for its ability to improve the robustness of neural networks. 

\section*{Acknowledgement}
The authors thank the support of the NSF-Simons Research Collaborations on the Mathematical and Scientific Foundations of Deep Learning (NSF grant 2031985), and the ONR MURI Program (ONR grant 503405-78051). 
The authors thank the feedback from anonymous reviewers, which has greatly improved our experimental results.
The authors thank Enrique Mallada, Jeremias Sulam, and Ambar Pal for their suggestions and comments at various stages of this research project, and also thank Kyle Poe for carefully reading the manuscript. 

\section*{Impact Statement}
While this work is mostly theoretical, it tackles the issue of nonrobustness of neural networks trained by gradient-based algorithms, which potentially leads to more robust, reliable, and trustworthy neural networks in many application domains.

\bibliographystyle{icml2024}
\bibliography{ref.bib}

\newpage
\onecolumn
\appendix
\section{Additional Discussion on Validating Conjecture \ref{conj_conv}}\label{app_add_conj}
\subsection{Estimating \texorpdfstring{$\mathrm{dist}(f_p,F)$}{dist(f\_p,F)}}
To estimate the distance $\mathrm{dist}(f_p,F)=\inf_{c>0}\sup_{\mathbb{S}^{D-1}}|cf_p(\boldsymbol{x})-F(\boldsymbol{x})|$ between a trained network $f_p$ and a classifier $F$ (or $F^{(p)}$), we first pick an $\hat{c}>0$ that yields the least-square match between $\hat{c}f_p(\boldsymbol{x})$ and $F(\boldsymbol{x})$ over a set of points sampled from $\mathrm{Unif}(\mathbb{S}^{D-1})$. Then given this choice of $\hat{c}$, we run normalized projected gradient ascent on $|\hat{c}f_p(\boldsymbol{x})-F(\boldsymbol{x})|^2$ starting from an initial $x$ sampled from $\mathrm{Unif}(\mathbb{S}^{D-1})$, and repeat a large number of times, the maximum value of $|\hat{c}f_p(\boldsymbol{x})-F(\boldsymbol{x})|$ at the end of Normalized PGA over all runs give an estimate of $\mathrm{dist}(f_p,F)$.

\begin{algorithm}
    \caption{Estimating $\mathrm{dist}(f_p,F)$}

    \textbf{Input}: Network $f_p$; Classifier $F$ (or $F^{(p)})$; step size $\alpha$; sample numbers $N_1$,  $N_2$;
    
    \textbf{Do}:
    \begin{enumerate}
        \item Sample $\boldsymbol{x}, i=1,\cdots,N_1$ from $\mathrm{Unif}(\mathbb{S}^{D-1})$; \qquad $\hat{c}\la \arg\min_{c>0}\sum_{i=1}^{N_1}|cf_p(\boldsymbol{x})-F(\boldsymbol{x})|^2$
        \item Sample new $\boldsymbol{x}, i=1,\cdots,N_2$ from $\mathrm{Unif}(\mathbb{S}^{D-1})$;

        $\ell(\cdot) \la |\hat{c}f_p(\cdot)-F(\cdot)|^2$;
        
        \textbf{For} $i=1,\cdots,N_2$:

        \qquad $\boldsymbol{x}^{(0)}\la \boldsymbol{x}$;
        
        \qquad \textbf{For} $t=1,\cdots, \mathsf{max\_iter}$:

        \qquad \qquad $\boldsymbol{x}^{(t)}\la \boldsymbol{x}^{(t-1)}+\alpha \frac{\nabla_x\ell(x^{(t-1)})}{\|\nabla_x\ell(x^{(t-1)})\|}$; \texttt{\% normalized gradient ascent}
        
        \qquad \qquad $\boldsymbol{x}^{(t)}\la \frac{\boldsymbol{x}^{(t)}}{\|\boldsymbol{x}^{(t)}\|}$; \texttt{\% projection onto unit sphere}

    \end{enumerate}

    \textbf{Return} $\max_{i}|\hat{c}f_p(\boldsymbol{x}^{(\mathsf{max\_iter})}_i)-F(\boldsymbol{x}^{(\mathsf{max\_iter})}_i)|$
    \label{algo_approx_model}
\end{algorithm}
\subsection{Additional experiment}
We conduct the same experiment as in Section \ref{ssec:val}, with more subclasses $K=20,K_1=8$, the results are still well aligned with Conjecture \ref{conj_conv}. Notably, as $K$ the total number of subclasses increases, the trained ReLU network is more vulnerable to $l_2$ attacks, compared to the case in Section \ref{ssec:val}, while generalized pReLU still is robust to these attacks.
\begin{figure*}[ht]
  \centering
  \includegraphics[width=\linewidth]{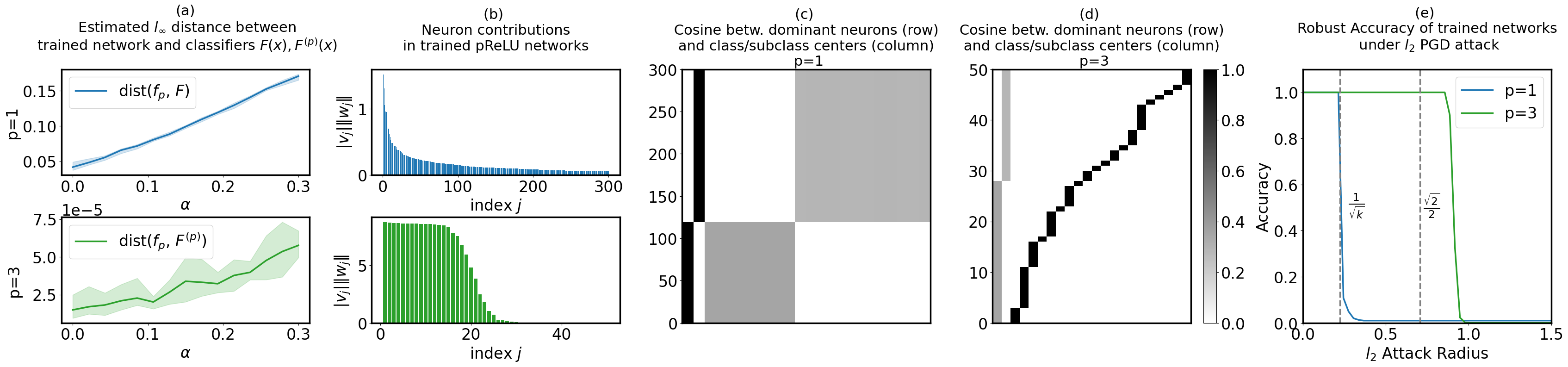}
  \vspace{-0.7cm}
  \caption{Additional Numerical experiment ($K=20,K_1=8$) validates Conjecture \ref{conj_conv}. \textbf{(a)} We train pReLU networks using SGD with small initialization, then estimate the distance $\text{dist}(f_p,F)$ between the trained network $f_p$ and the classifier $F$, when $p=1$ (top plot); When $p=3$, we estimate $\text{dist}(f_p,F^{(p)})$ instead (bottom plot). The training is done under different choices of intra-subclass variance $\alpha$ and repeated 10 times per $\alpha$; the Solid line shows the average and the shade denotes the region between max and min values. \textbf{(b)} Given a trained network obtained from an instance of this training ($\alpha=0.1$), we reorder the neurons w.r.t. their contributions $|v_j|\|\boldsymbol{w}_j\|$ and then plot the contributions in a bar plot; \textbf{(c)(d)} Given neurons with large contributions, we plot a colormap, with each pixel represents some  $\cos(\boldsymbol{w}_j,\boldsymbol{\mu})$, where $\boldsymbol{\mu}$ could be either average class centers $\bar{\boldsymbol{\mu}}_+$ and $\bar{\boldsymbol{\mu}}_-$ or subclass centers $\boldsymbol{\mu}_k,k\in[K]$. Note: For visibility, the neurons are reordered again so that neurons aligned with the same $\boldsymbol{\mu}$ are grouped together. \textbf{(e)} Lastly, we carry out $l_2$ PGD attack on a test dataset and plot the robust accuracy of the trained network under different choices of attack radius.}
  \label{fig_conjecture_val_addi}
\end{figure*}

\newpage
\section{Additional experiments on MNIST dataset}
\label{app_mnist}
To further illustrate the effectiveness of pReLU activation in improving the adversarial robustness of the trained network, we conduct the following additional experiments on MNIST dataset:

\myparagraph{MNIST digits classification} We follow the same experiment setting in Section \ref{ssec:mnist} and train the network to classify the digits instead of their parity. Figure \ref{fig_mnist_training_10_classes} reports the training results. Despite that our theoretical results only suggest robustness gain can be achieved under datasets with subclasses, these additional experimental results show that the pReLU networks are still much more robust than their ReLU counterpart even for the original digit classification task, which clearly does not follow our data assumption.

\begin{figure*}[ht]
  \centering
  \vspace{-0.2cm}
  \includegraphics[width=\linewidth]{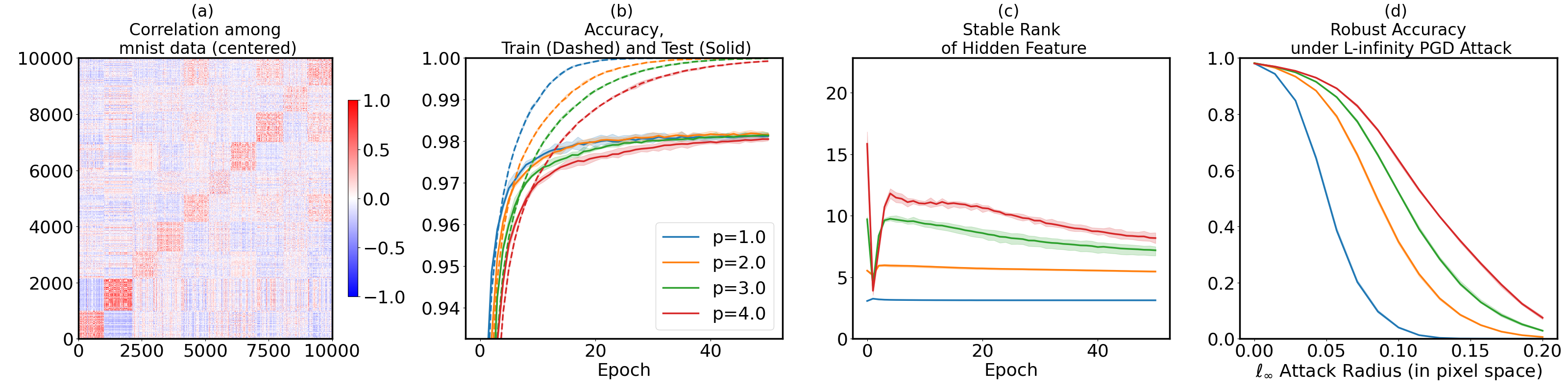}
  \vspace{-0.9cm}
  \caption{Digits classification on MNIST dataset with pReLU networks. \textbf{(a)} We plot the data correlation as a colormap, where each pixel represents some $\cos\lp x_i,x_j\rp$ between two centered data $x_i,x_j$ from MNIST training dataset; \textbf{(b)} We run Adam with batch size $1000$ to train a pReLU network under Kaiming initialization (repeated 10 times), then plot the training/testing accuracy during training for different choice of $p$; \textbf{(c)} We stack the hidden post-activation representation of each training sample into a matrix and compute its stable rank, and plot the evolution of this stable rank during training; \textbf{(d)} After training for $50$ epoch, we carry out APGD $l_\infty$-attack on MNIST test dataset (in pixel space) and plot the robust accuracy of the trained pReLU network under different choice of attack radius.} 
  \label{fig_mnist_training_10_classes}
\end{figure*}

\myparagraph{Evaluating robustness under different attacks} Lastly, we check the adversarial robustness of the trained networks above under attacks of different norms. The results are summarized in the table below, and they show that the robustness gained from pReLU activation is agnostic to the choice of attacks.
\begin{table}[ht]
\caption{Robust accuracy of pReLU networks under different attacks. The reported accuracy is the mean value over 5 runs with random initialization. Bold text indicates the best accuracy within the same row.}
\begin{tabular}{l|llll}
 & ReLU ($p=1$) & ReLU ($p=2$) & ReLU ($p=3$) & ReLU ($p=4$) \\ \hline
 Clean Acc.& 0.981($\pm 0.000$) & \textbf{0.982} ($\pm 0.000$)& \textbf{0.982} ($\pm 0.000$)& 0.980 ($\pm 0.000$)\\ \hline
 Robust Acc. ($\ell_\infty$, radius$=0.05$)& 0.512 ($\pm 0.006$)& 0.844 ($\pm 0.002$)& 0.892 ($\pm 0.002$)& \textbf{0.913} ($\pm 0.001$)\\
 Robust Acc. ($\ell_\infty$, radius$=0.1$)& 0.040 ($\pm 0.002$)& 0.346 ($\pm 0.005$)& 0.522 ($\pm 0.004$)& \textbf{0.637} ($\pm 0.004$)\\ \hline
 Robust Acc. ($\ell_2$, radius$=1$)& 0.301 ($\pm 0.004$)& 0.640 ($\pm 0.006$)& 0.726 ($\pm 0.004$)& \textbf{0.775} ($\pm 0.002$)\\
 Robust Acc. ($\ell_2$, radius$=2$)& 0.007 ($\pm 0.001$)& 0.101 ($\pm 0.002$)& 0.165 ($\pm 0.003$)& \textbf{0.239} ($\pm 0.003$)\\ \hline
 Robust Acc. ($\ell_1$, radius$=5$)& 0.500 ($\pm 0.007$)& 0.744 ($\pm 0.007$)& 0.780 ($\pm 0.004$)& \textbf{0.807} ($\pm 0.002$)\\
 Robust Acc. ($\ell_1$, radius$=10$)& 0.098 ($\pm 0.004$)& 0.294 ($\pm 0.002$)& 0.354 ($\pm 0.002$)& \textbf{0.402} ($\pm 0.003$)
\end{tabular}
\label{tabel_mnist}
\end{table}

\newpage

\section{Proofs for Proposition \ref{prop_gen} and Theorem \ref{thm_robust}}\label{app_pf_thm}
\setcounter{subsection}{-1}
\subsection{Verifying the claim regarding realizability of $F(\mathbf{x})$ and $F^{(p)}(\mathbf{x})$}
Before we present our poofs for the main results, we start with verifying our claim regarding realizability of $F(\boldsymbol{x})$ and $F^{(p)}(\boldsymbol{x})$:
\begin{claim}[restated]
    The following two statements are true:
    \begin{itemize}[leftmargin=9pt,topsep=0pt,parsep=0pt]
       \item 
        When $p=1$, let $I_+$ and $I_-$ be any two non-empty and disjoint subsets of $[h]$, let $\boldsymbol{\lambda}:=(\lambda_j) \in \mathbb{R}^h_{\geq 0}$ be any vector such that $\sum_{j\in I_+}\lambda_j^2=\sqrt{K_1}$ and $\sum_{j\in I_-}\lambda_j^2=\sqrt{K_2}$, and let
        \begin{align*}
        & \boldsymbol{w}_j=\lambda_j\bar{\boldsymbol{\mu}}_+,\ v_j=\lambda_j, &\forall j\in I_+\,,\\
        & \boldsymbol{w}_j=\lambda_j\bar{\boldsymbol{\mu}}_-,\ v_j=-\lambda_j, &\forall j\in I_-\,,\\
        & \boldsymbol{w}_j = 0,\  v_j = 0, &\text{otherwise}. 
        \end{align*}
        Then $f_p(x;\{\boldsymbol{w}_j,v_j\}_{j=1}^h)\equiv F(\boldsymbol{x})\,.$
        \item When $p\geq 1$, let $I_1,\cdots, I_K$ be any $K$ non-empty and disjoint subsets of $[h]$, let $\boldsymbol{\lambda}:=(\lambda_j) \in \mathbb{R}^h_{\geq 0}$ be any vector such that $\sum_{j\in I_k}\lambda_j^2=1,\forall k\in[K]$, and let 
        \begin{align*}
        & \boldsymbol{w}_j=\lambda_j\boldsymbol{\mu}_k,\ v_j=\lambda_j, &\forall j\in I_k, k\leq K_1\,,\\
        & \boldsymbol{w}_j=\lambda_j\boldsymbol{\mu}_k,\ v_j=-\lambda_j, &\forall j\in I_k, k>K_1\,,\\
        & \boldsymbol{w}_j = 0,\ v_j = 0 & \text{otherwise}. 
        \end{align*}
        Then $f_p(x;\{\boldsymbol{w}_j,v_j\}_{j=1}^h)\equiv F^{(p)}(\boldsymbol{x})\,.$
        \end{itemize}
\end{claim}
\begin{proof}
    \myparagraph{Statement 1: When $p=1$}. Let $I_+,I_-, \boldsymbol{\lambda}$ be define as in the statement, then we have
    \begin{align*}
        f_p(\boldsymbol{x};\{\boldsymbol{w}_j,v_j\}_{j=1}^h) &=\; \sum_{j=1}^hv_j\sigma(\lan \boldsymbol{x},\boldsymbol{w}_j\ran)\\
        &=\; \sum_{j\in I_+}v_j\sigma(\lan \boldsymbol{x},\boldsymbol{w}_j\ran)+\sum_{j\in I_-}v_j\sigma(\lan \boldsymbol{x},\boldsymbol{w}_j\ran)\\
        &=\; \sum_{j\in I_+}\lambda_j\sigma(\lan \boldsymbol{x},\lambda_j\bar{\boldsymbol{\mu}}_+\ran)-\sum_{j\in I_-}\lambda_j\sigma(\lan \boldsymbol{x},\lambda_j\bar{\boldsymbol{\mu}}_-\ran)\\
        &=\; \sum_{j\in I_+}\lambda_j^2\sigma(\lan \boldsymbol{x},\bar{\boldsymbol{\mu}}_+\ran)-\sum_{j\in I_-}\lambda_j^2\sigma(\lan \boldsymbol{x},\bar{\boldsymbol{\mu}}_-\ran)\\
        &=\; \sqrt{K_1}\sigma(\lan \boldsymbol{x},\bar{\boldsymbol{\mu}}_+\ran)-\sqrt{K_2}\sigma(\lan \boldsymbol{x},\bar{\boldsymbol{\mu}}_-\ran)=F(\boldsymbol{\boldsymbol{x}})\,.
    \end{align*}
    \myparagraph{Statement 2: When $p\geq 1$}. Let $\{I_k\}^K_{k=1},\boldsymbol{\lambda}$ be define as in the statement, then we have
    \begin{align*}
        f_p(\boldsymbol{x};\{\boldsymbol{w}_j,v_j\}_{j=1}^h) &=\; \sum_{j=1}^hv_j\frac{\sigma^p(\lan \boldsymbol{x},\boldsymbol{w}_j\ran)}{\|\boldsymbol{w}_j\|^{p-1}}\\
        &=\; \sum_{k\leq K_1}\sum_{j\in I_k}v_j\frac{\sigma^p(\lan \boldsymbol{x},\boldsymbol{w}_j\ran)}{\|\boldsymbol{w}_j\|^{p-1}}+\sum_{k> K_1}\sum_{j\in I_k}v_j\frac{\sigma^p(\lan \boldsymbol{x},\boldsymbol{w}_j\ran)}{\|\boldsymbol{w}_j\|^{p-1}}\\
        &=\; \sum_{k\leq K_1}\sum_{j\in I_k}\lambda_j\frac{\sigma^p(\lan \boldsymbol{x},\lambda_j \boldsymbol{\mu}_k\ran)}{\|\lambda_j\boldsymbol{\mu}_k\|^{p-1}}-\sum_{k> K_1}\sum_{j\in I_k}\lambda_j\frac{\sigma^p(\lan \boldsymbol{x},\lambda_j\boldsymbol{\mu}_k\ran)}{\|\lambda_j\boldsymbol{\mu}_k\|^{p-1}}\\
        &=\; \sum_{k\leq K_1}\sum_{j\in I_k}\lambda_j^2\sigma^p(\lan \boldsymbol{x},\boldsymbol{\mu}_k\ran)-\sum_{k> K_1}\sum_{j\in I_k}\lambda_j^2\sigma^p(\lan \boldsymbol{x},\boldsymbol{\mu}_k\ran)\\
        &=\; \sum_{k\leq K_1}\sigma^p(\lan \boldsymbol{x},\boldsymbol{\mu}_k\ran)-\sum_{k> K_1}\sigma^p(\lan \boldsymbol{x},\boldsymbol{\mu}_k\ran)=F^{(p)}(\boldsymbol{\boldsymbol{x}})\,.
    \end{align*}
\end{proof}
\subsection{Auxiliary lemmas}
The proof will use some basic facts in probability theory and we list them below.
\begin{lemma}\label{lem_split_prob}
    Let $\mathcal{E}_1,\mathcal{E}_2$ be two events defined on some probability space, then
    \be
        \prob(\mathcal{E}_1)\leq \prob(\mathcal{E}_1\cap \mathcal{E}_2)+\prob(\mathcal{E}_2^c)
    \ee
\end{lemma}
\begin{proof}
    $\prob(\mathcal{E}_1)=\prob\lp\mathcal{E}_1\cap\lp \mathcal{E}_2\cup \mathcal{E}_2^c\rp\rp=\prob\lp\lp \mathcal{E}_1\cap\mathcal{E}_2\rp\cup\lp \mathcal{E}_1\cap\mathcal{E}_2^c\rp\rp\leq \prob(\mathcal{E}_1\cap \mathcal{E}_2)+\prob(\mathcal{E}_1\cap\mathcal{E}_2^c)\leq \prob(\mathcal{E}_1\cap \mathcal{E}_2)+\prob(\mathcal{E}_2^c)$.
\end{proof}
\begin{lemma}[Hoeffding inequality]\label{lem_heoffding}
    For any unit vector $\boldsymbol{\mu}\in\mathbb{S}^{D-1}$, we have
    \be
        \prob_{\boldsymbol{\varepsilon}\sim \mathcal{N}\lp 0,\frac{\var^2}{D}\boldsymbol{I}_D\rp}\lp |\lan \boldsymbol{\mu},\boldsymbol{\varepsilon}\ran|>t\rp\leq 2\exp\lp -\frac{CDt^2}{\var^2}\rp\,,
    \ee
    for some constant $C>0$.
\end{lemma}
\subsection{Proof for Proposition \ref{prop_gen}}
\begin{customprop}{1}[Test error on clean data, restated] 
    Given classifiers $F(\boldsymbol{x})$, $F^{(p)}(\boldsymbol{x})$ defined in \eqref{eq_f0},\eqref{eq_fp}, and a test sample $(x,y)\sim \mathcal{D}_{X,Y}$, we have, for some constant $C>0$, 
    \begin{align*}
        &\;\prob\lp F(\boldsymbol{x})y>0\rp \geq 1-4\exp\lp -\frac{CD}{4\var^2K}\rp\\
        &\;\prob\lp F^{(p)}(\boldsymbol{x})y>0\rp\geq 1-2(K+1)\exp\lp -\frac{CD}{\var^2K^2}\rp,\ \forall p\geq 1\,.
    \end{align*}
\end{customprop}
\begin{proof}
    The proof is split into parts: We first prove the bound for $F(\boldsymbol{x})$, then show the one for $F^{(p)}(\boldsymbol{x})$. But in both cases we have 
    \be
        \prob_{(\boldsymbol{x},y)\sim \mathcal{D}_{X,Y}}\lp G(\boldsymbol{x})y>0\rp = \sum_{k=1}^K\prob_{(\boldsymbol{x},y)\sim \mathcal{D}_{X,Y}}\lp G(\boldsymbol{x})y>0 \mid z=k\rp\prob\lp z=k\rp\,,
    \ee
    where $G$ can be either $F$ or $F^{(p)}$. Thus, it suffices to show
    \be
        \prob_{(\boldsymbol{x},y)\sim \mathcal{D}_{X,Y}}\lp G(\boldsymbol{x})y>0 \mid z=k\rp\geq 1-4\exp\lp -\frac{CD}{\var^2K}\rp, \forall k=1,\cdots,K\,.
    \ee
    \paragraph{Bound for $F(\boldsymbol{x})$} we start with the case of $G$ being $F$. When $k\leq K_1$, we have 
    \be
        \prob_{(\boldsymbol{x},y)\sim \mathcal{D}_{X,Y}}\lp F(\boldsymbol{x})y>0 \mid z=k\rp=\prob_{\boldsymbol{\varepsilon}\sim \mathcal{N}\lp 0,\frac{\var^2}{D}\boldsymbol{I}_D\rp}\lp F(\boldsymbol{\mu}_k+\boldsymbol{\varepsilon})>0\rp\,,
    \ee
    and then,
    \begin{align}
        &\;\prob_{\boldsymbol{\varepsilon}\sim \mathcal{N}\lp 0,\frac{\var^2}{D}\boldsymbol{I}_D\rp}\lp F(\boldsymbol{\mu}_k+\boldsymbol{\varepsilon})>0\rp \nonumber\\
        =&\;1-\prob_{\boldsymbol{\varepsilon}\sim \mathcal{N}\lp 0,\frac{\var^2}{D}\boldsymbol{I}_D\rp}\lp F(\boldsymbol{\mu}_k+\boldsymbol{\varepsilon})<0\rp\nonumber\\
        =&\;1-\prob_{\boldsymbol{\varepsilon}\sim \mathcal{N}\lp 0,\frac{\var^2}{D}\boldsymbol{I}_D\rp}\lp \sqrt{K_1}\act\lp \frac{1}{\sqrt{K_1}}+\lan \boldsymbol{\varepsilon},\bar{\boldsymbol{\mu}}_+\ran\rp-\sqrt{K_2}\act\lp \lan \boldsymbol{\varepsilon},\bar{\boldsymbol{\mu}}_-\ran\rp<0, \mathcal{E}_1\rp \nonumber\\
        \geq &\;1-\prob_{\boldsymbol{\varepsilon}\sim \mathcal{N}\lp 0,\frac{\var^2}{D}\boldsymbol{I}_D\rp}\lp \sqrt{K_1}\lp \frac{1}{\sqrt{K_1}}-\lvt \lan \boldsymbol{\varepsilon},\bar{\boldsymbol{\mu}}_+\ran\rvt \rp-\sqrt{K_2} \lvt \lan \boldsymbol{\varepsilon},\bar{\boldsymbol{\mu}}_-\ran\rvt<0\rp \label{app_eq_ref_1}\\
        = &\;1-\prob_{\boldsymbol{\varepsilon}\sim \mathcal{N}\lp 0,\frac{\var^2}{D}\boldsymbol{I}_D\rp}\lp 1-\sqrt{K_1}\lvt \lan \boldsymbol{\varepsilon},\bar{\boldsymbol{\mu}}_+\ran\rvt -\sqrt{K_2} \lvt \lan \boldsymbol{\varepsilon},\bar{\boldsymbol{\mu}}_-\ran\rvt<0\rp \nonumber\\
        \geq &\;1-\prob_{\boldsymbol{\varepsilon}\sim \mathcal{N}\lp 0,\frac{\var^2}{D}\boldsymbol{I}_D\rp}\lp \lvt \lan \boldsymbol{\varepsilon},\bar{\boldsymbol{\mu}}_+\ran\rvt>\frac{1}{2\sqrt{K_1}}\rp  -\prob_{\boldsymbol{\varepsilon}\sim \mathcal{N}\lp 0,\frac{\var^2}{D}\boldsymbol{I}_D\rp}\lp \lvt \lan \boldsymbol{\varepsilon},\bar{\boldsymbol{\mu}}_-\ran\rvt>\frac{1}{2\sqrt{K_2}}\rp \nonumber\\
        \geq &\;1-2\exp\lp -\frac{CD}{4\var^2 K_1}\rp -2\exp\lp -\frac{CD}{4\var^2 K_2}\rp  \leq 1-4\exp\lp -\frac{CD}{4\var^2K}\rp\,,
    \end{align}
    where \eqref{app_eq_ref_1} uses the fact that $\act(x)$ is non-decreasing w.r.t. $x$., and that $\act(x)\leq |x|$. The last line uses Lemma \ref{lem_heoffding}. The proof of the case $k>K_1$ is identical to the one above by the symmetry of the problem. 
    \paragraph{Bound for $F^{(p)}(\boldsymbol{x})$}
    When $k\leq K_1$, we have, again,
    \be
        \prob_{(\boldsymbol{x},y)\sim \mathcal{D}_{X,Y}}\lp F^{(p)}(\boldsymbol{x})y>0 \mid z=k\rp=\prob_{\boldsymbol{\varepsilon}\sim \mathcal{N}\lp 0,\frac{\var^2}{D}\boldsymbol{I}_D\rp}\lp F^{(p)}(\boldsymbol{\mu}_k+\boldsymbol{\varepsilon})>0\rp\,,
    \ee
    we define the event $\mathcal{E}_1:=\lb|\lan \boldsymbol{\mu}_k, \boldsymbol{\varepsilon}\ran|<1\rb$. Then by Lemma \ref{lem_split_prob},
    \begin{align}
        &\;\prob_{\boldsymbol{\varepsilon}\sim \mathcal{N}\lp 0,\frac{\var^2}{D}\boldsymbol{I}_D\rp}\lp F^{(p)}(\boldsymbol{\mu}_k+\boldsymbol{\varepsilon})>0\rp \nonumber\\
        =&\;1-\prob_{\boldsymbol{\varepsilon}\sim \mathcal{N}\lp 0,\frac{\var^2}{D}\boldsymbol{I}_D\rp}\lp F^{(p)}(\boldsymbol{\mu}_k+\boldsymbol{\varepsilon})<0\rp\nonumber\\
        \leq &\;1-\prob_{\boldsymbol{\varepsilon}\sim \mathcal{N}\lp 0,\frac{\var^2}{D}\boldsymbol{I}_D\rp}\lp F^{(p)}(\boldsymbol{\mu}_k+\boldsymbol{\varepsilon})<0, \mathcal{E}_1\rp -\prob\lp \mathcal{E}_1^c\rp\nonumber\\
        \leq &\;1-\prob_{\boldsymbol{\varepsilon}\sim \mathcal{N}\lp 0,\frac{\var^2}{D}\boldsymbol{I}_D\rp}\lp 
        \act^p\lp 1+\lan \boldsymbol{\mu}_k,\boldsymbol{\varepsilon}\ran\rp +\sum_{l\leq K_1,l\neq k}\act^p\lp \lan \boldsymbol{\mu}_l,\boldsymbol{\varepsilon}\ran\rp -\sum_{K_1<l\leq K_2}\act^p\lp \lan \boldsymbol{\mu}_l,\boldsymbol{\varepsilon}\ran\rp<0
        , \mathcal{E}_1\rp -\prob\lp \mathcal{E}_1^c\rp\nonumber\\
        \leq &\;1-\prob_{\boldsymbol{\varepsilon}\sim \mathcal{N}\lp 0,\frac{\var^2}{D}\boldsymbol{I}_D\rp}\lp 
        \lp 1-\lvt \lan \boldsymbol{\mu}_k,\boldsymbol{\varepsilon}\ran\rvt\rp^p-\sum_{l\neq k}\lvt \lan \boldsymbol{\mu}_l,\boldsymbol{\varepsilon}\ran\rvt^p<0
        , \mathcal{E}_1\rp -\prob\lp \mathcal{E}_1^c\rp\label{app_eq_ref_2}\\
        \leq &\;1-\prob_{\boldsymbol{\varepsilon}\sim \mathcal{N}\lp 0,\frac{\var^2}{D}\boldsymbol{I}_D\rp}\lp 
        1-\lvt \lan \boldsymbol{\mu}_k,\boldsymbol{\varepsilon}\ran\rvt-\lp\sum_{l\neq k}\lvt \lan \boldsymbol{\mu}_l,\boldsymbol{\varepsilon}\ran\rvt^p\rp^{1/p}<0, \mathcal{E}_1
        \rp -\prob\lp \mathcal{E}_1^c\rp\nonumber\\
        \leq &\;1-\prob_{\boldsymbol{\varepsilon}\sim \mathcal{N}\lp 0,\frac{\var^2}{D}\boldsymbol{I}_D\rp}\lp 
        1-\lvt \lan \boldsymbol{\mu}_k,\boldsymbol{\varepsilon}\ran\rvt-\sum_{l\neq k}\lvt \lan \boldsymbol{\mu}_l,\boldsymbol{\varepsilon}\ran\rvt<0, \mathcal{E}_1
        \rp -\prob\lp \mathcal{E}_1^c\rp\label{app_eq_ref_3}\\
        \leq &\;1-\prob_{\boldsymbol{\varepsilon}\sim \mathcal{N}\lp 0,\frac{\var^2}{D}\boldsymbol{I}_D\rp}\lp 
        1-\lvt \lan \boldsymbol{\mu}_k,\boldsymbol{\varepsilon}\ran\rvt-\sum_{l\neq k}\lvt \lan \boldsymbol{\mu}_l,\boldsymbol{\varepsilon}\ran\rvt<0
        \rp -\prob\lp \mathcal{E}_1^c\rp\nonumber\\
        \leq &\;1-\sum_{k=1}^K\prob_{\boldsymbol{\varepsilon}\sim \mathcal{N}\lp 0,\frac{\var^2}{D}\boldsymbol{I}_D\rp}\lp \lvt \lan \boldsymbol{\mu}_k,\boldsymbol{\varepsilon}\ran\rvt>\frac{1}{K}\rp -\prob\lp \mathcal{E}_1^c\rp\nonumber\\
        \leq &\;1-2K\exp\lp -\frac{CD}{\var^2 K^2}\rp-2\exp\lp -\frac{CD}{\var^2}\rp\leq 1-2(K+1)\exp\lp -\frac{CD}{\var^2 K^2}\rp\,,
    \end{align}
    where \eqref{app_eq_ref_2} uses the fact that under the event $\mathcal{E}_1$, one has $\act^p\lp 1+\lan \boldsymbol{\mu}_k,\boldsymbol{\varepsilon}\ran\rp\geq \lp 1-\lvt \lan \boldsymbol{\mu}_k,\boldsymbol{\varepsilon}\ran\rvt\rp^p$, and \eqref{app_eq_ref_3} uses the fact that $\|x\|_p\leq \|x\|_1$ for any vector $x$ and $p\geq 1$. The last line uses Lemma \ref{lem_heoffding}. The proof of the case $k>K_1$ is identical to the one above by the symmetry of the problem. 
\end{proof}
\subsection{Proof for Theorem \ref{thm_robust}}
\begin{customthm}{1}[$l_2$-Adversarial Robustness, restated]
    Given classifiers $F(\boldsymbol{x})$, $F^{(p)}(\boldsymbol{x})$ defined in \eqref{eq_f0},\eqref{eq_fp}, and a test sample $(x,y)\sim \mathcal{D}_{X,Y}$, the following statement is true for some constant $C>0$:
    \begin{itemize}[leftmargin=9pt,topsep=0pt,parsep=0pt]
        \item $\lp \text{Non-robustness of } F(\boldsymbol{x}) \text{ against }\mathcal{O}\lp\frac{1}{\sqrt{K}}\rp \text{-attack} \rp$ We let $\boldsymbol{d}_0:=\frac{\sqrt{K_1}\bar{\boldsymbol{\mu}}_+-\sqrt{K_2}\bar{\boldsymbol{\mu}}_-}{\sqrt{K}}\in\mathbb{S}^{D-1}$, then for some $\rho>0$,
        \begin{align*}
            &\;\prob\!\lp F\!\lp \boldsymbol{x}-\frac{y(1+\rho)}{\sqrt{K}}\boldsymbol{d}_0\rp y\!>\!0\rp \!\leq\! 2\exp\!\lp -\frac{CD\rho^2}{K\var^2}\rp\!.
        \end{align*}
        \item $\lp \text{Robustness of } F^{(p)}(\boldsymbol{x}) \text{ against }\mathcal{O}\lp 1\rp \text{-attack }\rp$ We let $p\geq 2$, then for some $0\leq\delta<\sqrt{2}$,
        \begin{align*}
            &\;\prob\lp \min_{\|\boldsymbol{d}\|\leq 1} F^{(p)}\lp \boldsymbol{x}+\frac{\sqrt{2}-\delta}{2}\boldsymbol{d}\rp y>0\rp\geq 1-2(K+1)\exp\lp -\frac{CD\delta^2}{2K^2\var^2}\rp\,.
        \end{align*}
    \end{itemize}    
\end{customthm}
\begin{proof}
    This proof is split into two parts. One for $F(\boldsymbol{x})$ and another for $F^{(p)}(\boldsymbol{x})$.
    \paragraph{Non-robustness of $F(\boldsymbol{x})$} Since
    \be
        \prob_{(\boldsymbol{x},y)\sim \mathcal{D}_{X,Y}}\lp F\lp \boldsymbol{x}-\frac{y(1+\rho)}{\sqrt{K}}\boldsymbol{d}_0\rp y>0\rp = \sum_{k=1}^K\prob_{(\boldsymbol{x},y)\sim \mathcal{D}_{X,Y}}\lp F\lp \boldsymbol{x}-\frac{y(1+\rho)}{\sqrt{K}}\boldsymbol{d}_0\rp y>0 \mid z=k\rp\prob\lp z=k\rp\,,
    \ee
    It suffices to show
    \be
        \prob_{(\boldsymbol{x},y)\sim \mathcal{D}_{X,Y}}\lp F\lp \boldsymbol{x}-\frac{y(1+\rho)}{\sqrt{K}}\boldsymbol{d}_0\rp y>0 \mid z=k \rp \leq 2\exp\lp -\frac{CD^2K}{(\rho-1)^2\var^2}\rp\,, \forall k\leq K
    \ee
    When $k\leq K_1$, we have
    \begin{align}
        &\;\prob_{(\boldsymbol{x},y)\sim \mathcal{D}_{X,Y}}\lp F\lp \boldsymbol{x}-\frac{y(1+\rho)}{\sqrt{K}}\boldsymbol{d}_0\rp y>0 \mid z=k \rp\nonumber\\
        =&\; \prob_{\boldsymbol{\varepsilon}\sim \mathcal{N}\lp 0,\frac{\var^2}{D}\boldsymbol{I}_D\rp} \lp F\lp \boldsymbol{\mu}_k +\boldsymbol{\varepsilon}-\frac{1+\rho}{\sqrt{K}}\boldsymbol{d}_0\rp >0 \rp\nonumber\\
        =&\; \prob_{\boldsymbol{\varepsilon}\sim \mathcal{N}\lp 0,\frac{\var^2}{D}\boldsymbol{I}_D\rp} \lp 
        \sqrt{K_1}\act\lp \frac{1}{\sqrt{K_1}}+\lan \boldsymbol{\varepsilon}, \bar{\boldsymbol{\mu}}_+\ran -\frac{(1+\rho) \sqrt{K_1}}{K}\rp -\sqrt{K_2}\act\lp \lan \boldsymbol{\varepsilon}, \bar{\boldsymbol{\mu}}_-\ran+ \frac{(1+\rho) \sqrt{K_2}}{K}\rp
        >0 \rp\,.\label{app_eq_robust_1}
    \end{align}
    To proceed, we define the event $\mathcal{E}_2:=\{ \lan \boldsymbol{\varepsilon}, \bar{\boldsymbol{\mu}}_+\ran-\frac{(1+\rho) \sqrt{K_1}}{\sqrt{K}}+\frac{1}{\sqrt{K_1}}\geq 0\}$, then by Lemma \ref{lem_split_prob}
    \begin{align}
        &\;\eqref{app_eq_robust_1}\nonumber\\
        \leq &\;\prob_{\boldsymbol{\varepsilon}\sim \mathcal{N}\lp 0,\frac{\var^2}{D}\boldsymbol{I}_D\rp} \lp 
        \sqrt{K_1}\act\lp \frac{1}{\sqrt{K_1}}+\lan \boldsymbol{\varepsilon}, \bar{\boldsymbol{\mu}}_+\ran -\frac{(1+\rho) \sqrt{K_1}}{K}\rp -\sqrt{K_2}\act\lp \lan \boldsymbol{\varepsilon}, \bar{\boldsymbol{\mu}}_-\ran+ \frac{(1+\rho) \sqrt{K_2}}{K}\rp
        >0,  \mathcal{E}_2\rp\nonumber\\
        &\; \qquad\qquad +\prob\lp \sqrt{K_1}\act\lp \frac{1}{\sqrt{K_1}}+\lan \boldsymbol{\varepsilon}, \bar{\boldsymbol{\mu}}_+\ran -\frac{(1+\rho) \sqrt{K_1}}{K}\rp -\sqrt{K_2}\act\lp \lan \boldsymbol{\varepsilon}, \bar{\boldsymbol{\mu}}_-\ran+ \frac{(1+\rho) \sqrt{K_2}}{K}\rp
        >0, \mathcal{E}_2^c\rp\nonumber\\
        =&\; \prob_{\boldsymbol{\varepsilon}\sim \mathcal{N}\lp 0,\frac{\var^2}{D}\boldsymbol{I}_D\rp} \lp 
        \sqrt{K_1}\act\lp \frac{1}{\sqrt{K_1}}+\lan \boldsymbol{\varepsilon}, \bar{\boldsymbol{\mu}}_+\ran -\frac{(1+\rho) \sqrt{K_1}}{K}\rp -\sqrt{K_2}\act\lp \lan \boldsymbol{\varepsilon}, \bar{\boldsymbol{\mu}}_-\ran+ \frac{(1+\rho) \sqrt{K_2}}{K}\rp
        >0,  \mathcal{E}_2\rp\label{app_eq_ref_4}\\
        \leq &\;\prob_{\boldsymbol{\varepsilon}\sim \mathcal{N}\lp 0,\frac{\var^2}{D}\boldsymbol{I}_D\rp} \lp 
        1+ \sqrt{K_1}\lan \boldsymbol{\varepsilon}, \bar{\boldsymbol{\mu}}_+\ran - \frac{(1+\rho) K_1}{K} -\sqrt{K_2} \lan \boldsymbol{\varepsilon}, \bar{\boldsymbol{\mu}}_-\ran - \frac{(1+\rho) K_2}{K} >0
        ,  \mathcal{E}_2\rp\label{app_eq_ref_5}\\
        \leq &\;\prob_{\boldsymbol{\varepsilon}\sim \mathcal{N}\lp 0,\frac{\var^2}{D}\boldsymbol{I}_D\rp} \lp 
        1+ \sqrt{K}\lvt\lan \boldsymbol{\varepsilon}, \boldsymbol{d}_0\ran\rvt - (1+\rho) >0\rp\nonumber\\
        \leq &\;\prob_{\boldsymbol{\varepsilon}\sim \mathcal{N}\lp 0,\frac{\var^2}{D}\boldsymbol{I}_D\rp} \lp 
        \lvt\lan \boldsymbol{\varepsilon}, \boldsymbol{d}_0\ran\rvt >\frac{\rho}{\sqrt{K}}\rp\leq 2\exp\lp -\frac{CD\rho^2}{K\var^2}\rp\,,
    \end{align}
    where \eqref{app_eq_ref_4} is because the second probability is $0$, and \eqref{app_eq_ref_5} uses again the monotonicity of ReLU $\act(x)$. The last line uses Lemma \ref{lem_heoffding}. The proof of the case $k>K_1$ is identical to the one above by the symmetry of the problem. 
    \paragraph{Robustness of $F^{(p)}(\boldsymbol{x})$}
    It suffices to show
    \be
        \prob_{(\boldsymbol{x},y)\sim \mathcal{D}_{X,Y}}\lp \min_{\|\boldsymbol{d}\|\leq 1} F^{(p)}\lp \boldsymbol{x}+\frac{\sqrt{2}-\delta}{2}d\rp y<0 \mid z=k \rp \leq 2(K_2+2)\exp\lp -\frac{CD\delta^2}{2(K_2+1)^2\var^2}\rp\,.
    \ee
    When $k\leq K_1$, we have
    \begin{align}
        &\;\prob_{(\boldsymbol{x},y)\sim \mathcal{D}_{X,Y}}\lp \min_{\|\boldsymbol{d}\|\leq 1} F^{(p)}\lp \boldsymbol{x}+\frac{\sqrt{2}-\delta}{2}d\rp y<0 \mid z=k \rp\nonumber\\
        =&\;\prob_{\boldsymbol{\varepsilon}\sim \mathcal{N}\lp 0,\frac{\var^2}{D}\boldsymbol{I}_D\rp}\lp \min_{\|\boldsymbol{d}\|\leq 1} F^{(p)}\lp \boldsymbol{\mu}_k+\boldsymbol{\varepsilon}+\frac{\sqrt{2}-\delta}{2}d\rp <0\rp\nonumber\\
        =&\; \prob_{\boldsymbol{\varepsilon}\sim \mathcal{N}\lp 0,\frac{\var^2}{D}\boldsymbol{I}_D\rp}\lp 
        \min_{\|\boldsymbol{d}\|\leq 1} 
        \lhp\act^p\lp 1+\lan \boldsymbol{\mu}_k, \boldsymbol{\varepsilon}\ran +\frac{\sqrt{2}-\delta}{2}\lan \boldsymbol{d},\boldsymbol{\mu}_k\ran\rp\right.\right.\nonumber\\
        &\; \qquad\qquad\qquad\qquad+ \sum_{l\neq k, 1\leq l\leq K_1} \act^p \lp \lan \boldsymbol{\mu}_l,\boldsymbol{\varepsilon}\ran +\frac{\sqrt{2}-\delta}{2}\lan \boldsymbol{d},\boldsymbol{\mu}_l\ran\rp\nonumber\\
        &\; \qquad\qquad\qquad\qquad- \left.\left.\sum_{K_1+1\leq l\leq K_2} \act^p \lp \lan \boldsymbol{\mu}_l,\boldsymbol{\varepsilon}\ran +\frac{\sqrt{2}-\delta}{2}\lan \boldsymbol{d},\boldsymbol{\mu}_l\ran\rp\rhp<0\rp\label{app_eq_robust_2}
    \end{align}
    To proceed, we define the event 
    \be 
    \mathcal{E}_3:=\lb 1+\lan \boldsymbol{\mu}_k, \boldsymbol{\varepsilon}\ran +\frac{\sqrt{2}-\delta}{2}\lan \boldsymbol{d},\boldsymbol{\mu}_k\ran\geq 0,\forall d\in\mathbb{S}^{D-1}\rb\,,\ee
    and for ease of presentation, we write
    \begin{align*}
        G(\boldsymbol{\varepsilon},d)&\;:=\act^p\lp 1+\lan \boldsymbol{\mu}_k, \boldsymbol{\varepsilon}\ran +\frac{\sqrt{2}-\delta}{2}\lan \boldsymbol{d},\boldsymbol{\mu}_k\ran\rp + \sum_{l\neq k, 1\leq l\leq K_1} \act^p \lp \lan \boldsymbol{\mu}_l,\boldsymbol{\varepsilon}\ran +\frac{\sqrt{2}-\delta}{2}\lan \boldsymbol{d},\boldsymbol{\mu}_l\ran\rp\\
        &\;\qquad\qquad\qquad\qquad\qquad\qquad\qquad\qquad -\sum_{l=K_1+1}^{K_2} \act^p \lp \lan \boldsymbol{\mu}_l,\boldsymbol{\varepsilon}\ran +\frac{\sqrt{2}-\delta}{2}\lan \boldsymbol{d},\boldsymbol{\mu}_l\ran\rp\,.
    \end{align*}
    Then, by Lemma \ref{lem_split_prob},
    \begin{align}
        \eqref{app_eq_robust_2}&\;= \prob_{\boldsymbol{\varepsilon}\sim \mathcal{N}\lp 0,\frac{\var^2}{D}\boldsymbol{I}_D\rp}\lp \min_{\|\boldsymbol{d}\|\leq 1} G(\boldsymbol{\varepsilon},d) <0\rp\nonumber\\
        &\;\leq \prob_{\boldsymbol{\varepsilon}\sim \mathcal{N}\lp 0,\frac{\var^2}{D}\boldsymbol{I}_D\rp}\lp \min_{\|\boldsymbol{d}\|\leq 1} G(\boldsymbol{\varepsilon},d) <0,\mathcal{E}_3\rp +\prob\lp \mathcal{E}_3^c\rp\,.\label{app_eq_robust_3}
    \end{align}
    Now under the event $\mathcal{E}_3$, we can lower bound $G(\boldsymbol{\varepsilon},d)$ by
    \begin{align}
        &\;G(\boldsymbol{\varepsilon},d)\nonumber\\
        =&\; \act^p\lp 1+\lan \boldsymbol{\mu}_k, \boldsymbol{\varepsilon}\ran +\frac{\sqrt{2}-\delta}{2}\lan \boldsymbol{d},\boldsymbol{\mu}_k\ran\rp + \sum_{l\neq k, 1\leq l\leq K_1} \act^p \lp \lan \boldsymbol{\mu}_l,\boldsymbol{\varepsilon}\ran +\frac{\sqrt{2}-\delta}{2}\lan \boldsymbol{d},\boldsymbol{\mu}_l\ran\rp\nonumber\\
        &\;\qquad\qquad\qquad\qquad\qquad\qquad\qquad\qquad -\sum_{l=K_1+1}^{K_2} \act^p \lp \lan \boldsymbol{\mu}_l,\boldsymbol{\varepsilon}\ran +\frac{\sqrt{2}-\delta}{2}\lan \boldsymbol{d},\boldsymbol{\mu}_l\ran\rp\,.\\
        \geq &\; \lp 1+\lan \boldsymbol{\mu}_k, \boldsymbol{\varepsilon}\ran +\frac{\sqrt{2}-\delta}{2}\lan \boldsymbol{d},\boldsymbol{\mu}_k\ran\rp^p - \sum_{l=K_1+1}^{K_2}  \lp \lvt\lan \boldsymbol{\mu}_l,\boldsymbol{\varepsilon}\ran\rvt +\frac{\sqrt{2}-\delta}{2}\lvt\lan \boldsymbol{d},\boldsymbol{\mu}_l\ran\rvt\rp^p\,.
    \end{align}
    Thus, we have
    \begin{align}
        \eqref{app_eq_robust_3}\nonumber\\
        \leq &\;\prob_{\boldsymbol{\varepsilon}\sim \mathcal{N}\lp 0,\frac{\var^2}{D}\boldsymbol{I}_D\rp}\lp \min_{\|\boldsymbol{d}\|\leq 1} \lp 1+\lan \boldsymbol{\mu}_k, \boldsymbol{\varepsilon}\ran +\frac{\sqrt{2}-\delta}{2}\lan \boldsymbol{d},\boldsymbol{\mu}_k\ran\rp^p -\!\! \sum_{l=K_1+1}^{K_2}  \lp \lvt\lan \boldsymbol{\mu}_l,\boldsymbol{\varepsilon}\ran\rvt +\frac{\sqrt{2}-\delta}{2}\lvt\lan \boldsymbol{d},\boldsymbol{\mu}_l\ran\rvt\rp^p <0,\mathcal{E}_3\rp\nonumber\\
        &\qquad\qquad\qquad+\prob\lp \mathcal{E}_3^c\rp\nonumber\\
        =&\; \prob_{\boldsymbol{\varepsilon}\sim \mathcal{N}\lp 0,\frac{\var^2}{D}\boldsymbol{I}_D\rp}\lp \min_{\|\boldsymbol{d}\|\leq 1} 1+\lan \boldsymbol{\mu}_k, \boldsymbol{\varepsilon}\ran +\frac{\sqrt{2}-\delta}{2}\lan \boldsymbol{d},\boldsymbol{\mu}_k\ran -\!\! \lp\sum_{l=K_1+1}^{K_2}  \lp \lvt\lan \boldsymbol{\mu}_l,\boldsymbol{\varepsilon}\ran\rvt +\frac{\sqrt{2}-\delta}{2}\lvt\lan \boldsymbol{d},\boldsymbol{\mu}_l\ran\rvt\rp^p\rp^\frac{1}{p} <0,\mathcal{E}_3\rp\nonumber\\
        &\qquad\qquad\qquad+\prob\lp \mathcal{E}_3^c\rp\nonumber\\
        \leq&\; \prob_{\boldsymbol{\varepsilon}\sim \mathcal{N}\lp 0,\frac{\var^2}{D}\boldsymbol{I}_D\rp}\lp \underbrace{\min_{\|\boldsymbol{d}\|\leq 1} 1 +\frac{\sqrt{2}-\delta}{2}\lan \boldsymbol{d},\boldsymbol{\mu}_k\ran - \frac{\sqrt{2}-\delta}{2}\lp\sum_{l=K_1+1}^{K_2}  \lp \lvt\lan \boldsymbol{d},\boldsymbol{\mu}_l\ran\rvt\rp^p\rp^\frac{1}{p}}_{M^*(\delta)}\right.\nonumber\\
        &\;\qquad\qquad\qquad\qquad\qquad\qquad\left.-\lp\sum_{l=K_1+1}^{K_2}  \lp \lvt\lan \boldsymbol{\mu}_l,\boldsymbol{\varepsilon}\ran\rvt \rp^p\rp^\frac{1}{p} -\lvt\lan \boldsymbol{\mu}_k, \boldsymbol{\varepsilon}\ran\rvt<0,\mathcal{E}_3\rp+\prob\lp \mathcal{E}_3^c\rp\label{app_eq_ref_6}\\
        \leq &\; \prob_{\boldsymbol{\varepsilon}\sim \mathcal{N}\lp 0,\frac{\var^2}{D}\boldsymbol{I}_D\rp}\lp \lvt\lan \boldsymbol{\mu}_k, \boldsymbol{\varepsilon}\ran\rvt+ \lp\sum_{l=K_1+1}^{K_2}  \lp \lvt\lan \boldsymbol{\mu}_l,\boldsymbol{\varepsilon}\ran\rvt \rp^p\rp^\frac{1}{p} >M^*(\delta)\rp+\prob\lp \mathcal{E}_3^c\rp\nonumber\\
        \leq &\; \prob_{\boldsymbol{\varepsilon}\sim \mathcal{N}\lp 0,\frac{\var^2}{D}\boldsymbol{I}_D\rp}\lp \lvt\lan \boldsymbol{\mu}_k, \boldsymbol{\varepsilon}\ran\rvt+ \sum_{l=K_1+1}^{K_2}   \lvt\lan \boldsymbol{\mu}_l,\boldsymbol{\varepsilon}\ran\rvt  >M^*(\delta)\rp+\prob\lp \mathcal{E}_3^c\rp\label{app_eq_ref_7}\\
        \leq &\; \prob_{\boldsymbol{\varepsilon}\sim \mathcal{N}\lp 0,\frac{\var^2}{D}\boldsymbol{I}_D\rp}\lp \lvt\lan \boldsymbol{\mu}_k, \boldsymbol{\varepsilon}\ran\rvt >\frac{M^*(\delta)}{K_2+1}\rp  +\sum_{l=K_1+1}^{K_2} \prob_{\boldsymbol{\varepsilon}\sim \mathcal{N}\lp 0,\frac{\var^2}{D}\boldsymbol{I}_D\rp}\lp   \lvt\lan \boldsymbol{\mu}_l,\boldsymbol{\varepsilon}\ran\rvt  >\frac{M^*(\delta)}{K_2+1}\rp+\prob\lp \mathcal{E}_3^c\rp\nonumber\\
        \leq &\; \prob_{\boldsymbol{\varepsilon}\sim \mathcal{N}\lp 0,\frac{\var^2}{D}\boldsymbol{I}_D\rp}\lp \lvt\lan \boldsymbol{\mu}_k, \boldsymbol{\varepsilon}\ran\rvt >\frac{M^*(\delta)}{K_2+1}\rp  +\sum_{l=K_1+1}^{K_2} \prob_{\boldsymbol{\varepsilon}\sim \mathcal{N}\lp 0,\frac{\var^2}{D}\boldsymbol{I}_D\rp}\lp   \lvt\lan \boldsymbol{\mu}_l,\boldsymbol{\varepsilon}\ran\rvt  >\frac{M^*(\delta)}{K_2+1}\rp\nonumber\\
        &\;\qquad\qquad\qquad\qquad\qquad+\prob_{\boldsymbol{\varepsilon}\sim \mathcal{N}\lp 0,\frac{\var^2}{D}\boldsymbol{I}_D\rp}\lp \lvt\lan \boldsymbol{\mu}_k, \boldsymbol{\varepsilon}\ran\rvt >1-\frac{\sqrt{2}}{2}\rp\nonumber\\
        \leq &\; 2(K_2+1)\exp\lp -\frac{CD(M^*(\delta))^2}{(K_2+1)^2\var^2}\rp+2\exp\lp -\frac{CD}{4\var^2}\rp\leq 2(K_2+2)\exp\lp -\frac{CD(M^*(\delta))^2}{(K_2+1)^2\var^2}\rp\,,
    \end{align}
    where \eqref{app_eq_ref_6} uses the sub-additivity of $\ell_p$ norm, and \eqref{app_eq_ref_7} uses again the inequality $\|\boldsymbol{x}\|_p\leq \|\boldsymbol{x}\|_1$ for any $\boldsymbol{x}$ and $p\geq 1$. The last line uses Lemma \ref{lem_heoffding}. The remaining thing is to show that $M^*(\delta)=\frac{\delta}{\sqrt{2}}$. First, by the property of $l_p$ norm (when $p\geq 2$), we have
    \be
        \lp\sum_{l=K_1+1}^{K_2}  \lp \lvt\lan \boldsymbol{d},\boldsymbol{\mu}_l\ran\rvt\rp^p\rp^\frac{1}{p} \leq \sqrt{\sum_{l=K_1+1}^{K_2}  \lvt\lan \boldsymbol{d},\boldsymbol{\mu}_l\ran\rvt^2}\,,
    \ee
    then 
    \begin{align}
        M^*(\delta)=&\;\min_{\|\boldsymbol{d}\|\leq 1} 1 -\frac{\sqrt{2}-\delta}{2}\lan \boldsymbol{d},\boldsymbol{\mu}_k\ran - \frac{\sqrt{2}-\delta}{2}\lp\sum_{l=K_1+1}^{K_2}  \lp \lvt\lan \boldsymbol{d},\boldsymbol{\mu}_l\ran\rvt\rp^p\rp^\frac{1}{p}\nonumber\\
        \geq &\;\min_{\|\boldsymbol{d}\|\leq 1} 1 -\frac{\sqrt{2}-\delta}{2}\lvt\lan \boldsymbol{d},\boldsymbol{\mu}_k\ran\rvt - \frac{\sqrt{2}-\delta}{2}\sqrt{\sum_{l=K_1+1}^{K_2}  \lvt\lan \boldsymbol{d},\boldsymbol{\mu}_l\ran\rvt^2}\nonumber\\
        \geq &\;\min_{\|\boldsymbol{d}\|\leq 1} 1 -\frac{\sqrt{2}-\delta}{2}\sqrt{2}\sqrt{\lvt\lan \boldsymbol{d},\boldsymbol{\mu}_k\ran\rvt^2+\sum_{l=K_1+1}^{K_2}  \lvt\lan \boldsymbol{d},\boldsymbol{\mu}_l\ran\rvt^2}\label{app_eq_ref_10}\\
        \geq &\; \min_{\|\boldsymbol{d}\|\leq 1} 1 -\frac{\sqrt{2}-\delta}{2}\sqrt{2} = \frac{\delta}{\sqrt{2}}\,,
        \label{app_eq_robust_4}
    \end{align}
    where \eqref{app_eq_ref_10} uses the fact that $\sqrt{a}+\sqrt{b}\leq \sqrt{2(a+b)}$ for any $a,b\geq 0$, and \eqref{app_eq_robust_4} uses the fact that $\boldsymbol{\mu}_1,\cdots,\boldsymbol{\mu}_K$ are an orthonormal basis, thus $\sqrt{\sum_{l=1}^K|\lan \boldsymbol{d},\boldsymbol{\mu}_l\ran|^2}\leq \|\boldsymbol{d}\|\leq 1$.
    
    We have proved $M^*(\delta)\geq \frac{\delta}{\sqrt{2}}$, and this lower bound can be attained by $\boldsymbol{d}^*=\frac{-\boldsymbol{\mu}_k+\boldsymbol{\mu}_l}{\sqrt{2}}$ for any $K_1+1\leq l\leq K_2$. Therefore $M^*(\delta)=\frac{\delta}{\sqrt{2}}$. Finally, we have
    \begin{align}
        \prob_{(\boldsymbol{x},y)\sim \mathcal{D}_{X,Y}}\lp \min_{\|\boldsymbol{d}\|\leq 1} F^{(p)}\lp \boldsymbol{x}+\frac{\sqrt{2}-\delta}{2}\boldsymbol{d}\rp y<0 \mid z=k \rp &\leq\; 2(K_2+2)\exp\lp -\frac{CD(M^*(\delta))^2}{(K_2+1)^2\var^2}\rp\nonumber\\
        &=\; 2(K_2+2)\exp\lp -\frac{CD\delta^2}{2(K_2+1)^2\var^2}\rp\nonumber\\
        &\leq\; 2(K+1)\exp\lp -\frac{CD\delta^2}{2K^2\var^2}\rp\,.
    \end{align}
    The proof of the case $k>K_1$ is identical to the one above by the symmetry of the problem.
\end{proof}
\subsection{Complementary results to Theorem \ref{thm_robust}}\label{app_robust_comp}
\begin{theorem}[$l_2$-Adversarial Robustness, complementary to Theorem \ref{thm_robust}]
    Given classifiers $F(\boldsymbol{x})$, $F^{(p)}(\boldsymbol{x})$ defined in \eqref{eq_f0},\eqref{eq_fp}, and a test sample $(x,y)\sim \mathcal{D}_{X,Y}$, the following statement is true for some constant $C>0$:
    \begin{itemize}[leftmargin=9pt,topsep=0pt,parsep=0pt]
        \item For some $0\leq \rho \leq 1$,
        \begin{align*}
            &\;\prob\lp \min_{\|\boldsymbol{d}\|\leq 1}F\lp \boldsymbol{x}+\frac{(1-\rho)}{\sqrt{K}}\boldsymbol{d}\rp y> 0\rp \geq1- 2\exp\lp -\frac{CD\rho^2}{K\var^2}\rp.
        \end{align*}
        \item For some $0<\delta$,
        \begin{align*}
            &\;\prob\lp \min_{\|\boldsymbol{d}\|\leq 1} F^{(p)}\lp \boldsymbol{x}+\frac{\sqrt{2}+\delta}{2}\boldsymbol{d}\rp y>0\rp\leq 4\exp\lp \frac{CD \delta^2}{8K^2\var^2}\rp\,.
        \end{align*}
    \end{itemize}    
\end{theorem}
The proof has the same spirit as those for Theorem \ref{thm_robust} so we state it briefly.
\begin{proof}
    \textbf{Robustness of $F(\boldsymbol{x})$, complementary result} It suffices to show that
    \be
        \prob_{(\boldsymbol{x},y)\sim \mathcal{D}_{X,Y}}\lp \min_{\|\boldsymbol{d}\|\leq 1} F\lp \boldsymbol{x}+\frac{(1-\rho)}{\sqrt{K}}\boldsymbol{x}\rp y< 0 \mid z=k \rp \leq 4\exp\lp -\frac{CD\rho^2}{K\var^2}\rp\,, \forall k\leq K\,.
    \ee
    When $k\leq K_1$, we have
    {\small
    \begin{align}
        &\;\prob_{(\boldsymbol{x},y)\sim \mathcal{D}_{X,Y}}\lp \min_{\|\boldsymbol{d}\|\leq 1} F\lp \boldsymbol{x}+\frac{(1-\rho)}{\sqrt{K}}\boldsymbol{x}\rp y< 0 \mid z=k \rp\nonumber\\
        =&\;\prob_{\boldsymbol{\varepsilon}\sim \mathcal{N}\lp 0,\frac{\var^2}{D}\boldsymbol{I}_D\rp}\lp \min_{\|\boldsymbol{d}\|\leq 1} F\lp \boldsymbol{\mu}_k+\boldsymbol{\varepsilon}+\frac{(1-\rho)}{\sqrt{K}}\boldsymbol{x}\rp <0\rp\nonumber\\
        =&\; \prob_{\boldsymbol{\varepsilon}\sim \mathcal{N}\lp 0,\frac{\var^2}{D}\boldsymbol{I}_D\rp}\lp \min_{\|\boldsymbol{d}\|\leq 1} 
        \lhp \sqrt{K_1}\act\lp \frac{1}{\sqrt{K_1}}+\lan\boldsymbol{\varepsilon},\bar{\boldsymbol{\mu}}_+\ran+\frac{1-\rho}{\sqrt{K}}\lan \boldsymbol{d},\bar{\boldsymbol{\mu}}_+\ran\rp -\sqrt{K_2}\act\lp \lan \boldsymbol{\varepsilon},\bar{\boldsymbol{\mu}}_-\ran+\frac{1-\rho}{\sqrt{K}}\lan \boldsymbol{d},\bar{\boldsymbol{\mu}}_-\ran\rp\rhp
        <0\rp\label{eq_robust_4}
    \end{align}
    }
    If we still let $d_0:=\frac{\sqrt{K_1}\bar{\boldsymbol{\mu}}_+-\sqrt{K_2}\bar{\boldsymbol{\mu}}_-}{\sqrt{K}}\in\mathbb{S}^{D-1}$, then by the fact that $|x|\geq \sigma(x)\geq x$ for any $x$, we have
    {\small
    \begin{align}
        &\eqref{eq_robust_4}\nonumber\\
        &\leq \prob_{\boldsymbol{\varepsilon}\sim \mathcal{N}\lp 0,\frac{\var^2}{D}\boldsymbol{I}_D\rp}\lp \min_{\|\boldsymbol{d}\|\leq 1} 
        \lhp 1+\sqrt{K_1}\lan\boldsymbol{\varepsilon},\bar{\boldsymbol{\mu}}_+\ran+\sqrt{K_1}\frac{1-\rho}{\sqrt{K}}\lan \boldsymbol{d},\bar{\boldsymbol{\mu}}_+\ran -\sqrt{K_2}\lvt\lan \boldsymbol{\varepsilon},\bar{\boldsymbol{\mu}}_-\ran\rvt-\sqrt{K_2}\frac{1-\rho}{\sqrt{K}}\lvt\lan \boldsymbol{d},\bar{\boldsymbol{\mu}}_-\ran\rvt\rhp
        <0\rp\nonumber\\
        &\leq \prob_{\boldsymbol{\varepsilon}\sim \mathcal{N}\lp 0,\frac{\var^2}{D}\boldsymbol{I}_D\rp}\lp \min_{\|\boldsymbol{d}\|\leq 1} 
        \lhp 1-\sqrt{K_1}\frac{1-\rho}{\sqrt{K}}\lvt\lan \boldsymbol{d},\bar{\boldsymbol{\mu}}_+\ran\rvt -\sqrt{K_2}\frac{1-\rho}{\sqrt{K}}\lvt\lan \boldsymbol{d},\bar{\boldsymbol{\mu}}_-\ran\rvt\rhp-\sqrt{K_2}\lvt\lan \boldsymbol{\varepsilon},\bar{\boldsymbol{\mu}}_-\ran\rvt-\sqrt{K_1}\lvt\lan\boldsymbol{\varepsilon},\bar{\boldsymbol{\mu}}_+\ran\rvt
        <0\rp\nonumber\\
        &\leq \prob_{\boldsymbol{\varepsilon}\sim \mathcal{N}\lp 0,\frac{\var^2}{D}\boldsymbol{I}_D\rp}\lp \min_{\|\boldsymbol{d}\|\leq 1} 
        \lhp 1- \frac{1-\rho}{\sqrt{K}} \sqrt{K_1+K_2}\sqrt{\lvt\lan \boldsymbol{d},\bar{\boldsymbol{\mu}}_+\ran\rvt^2+\lvt\lan \boldsymbol{d},\bar{\boldsymbol{\mu}}_-\ran\rvt^2}\rhp-\sqrt{K_2}\lvt\lan \boldsymbol{\varepsilon},\bar{\boldsymbol{\mu}}_-\ran\rvt-\sqrt{K_1}\lvt\lan\boldsymbol{\varepsilon},\bar{\boldsymbol{\mu}}_+\ran\rvt
        <0\rp\label{app_eq_ref_11}\\
        &\leq \prob_{\boldsymbol{\varepsilon}\sim \mathcal{N}\lp 0,\frac{\var^2}{D}\boldsymbol{I}_D\rp}\lp \min_{\|\boldsymbol{d}\|\leq 1} 
        \lhp 1- (1-\rho)\rhp-\sqrt{K_2}\lvt\lan \boldsymbol{\varepsilon},\bar{\boldsymbol{\mu}}_-\ran\rvt-\sqrt{K_1}\lvt\lan\boldsymbol{\varepsilon},\bar{\boldsymbol{\mu}}_+\ran\rvt
        <0\rp\label{app_eq_ref_12}\\
        &\leq \prob_{\boldsymbol{\varepsilon}\sim \mathcal{N}\lp 0,\frac{\var^2}{D}\boldsymbol{I}_D\rp}\lp \sqrt{K_2}\lvt\lan \boldsymbol{\varepsilon},\bar{\boldsymbol{\mu}}_-\ran\rvt+\sqrt{K_1}\lvt\lan\boldsymbol{\varepsilon},\bar{\boldsymbol{\mu}}_+\ran\rvt
        >\rho \rp\nonumber\\
        &\leq \prob_{\boldsymbol{\varepsilon}\sim \mathcal{N}\lp 0,\frac{\var^2}{D}\boldsymbol{I}_D\rp}\lp \lvt\lan \boldsymbol{\varepsilon},\bar{\boldsymbol{\mu}}_-\ran\rvt+\lvt\lan\boldsymbol{\varepsilon},\bar{\boldsymbol{\mu}}_+\ran\rvt
        >\frac{\rho}{\sqrt{K}} \rp\nonumber\\
        &\leq 2\prob_{\boldsymbol{\varepsilon}\sim \mathcal{N}\lp 0,\frac{\var^2}{D}\boldsymbol{I}_D\rp}\lp \lvt\lan \boldsymbol{\varepsilon},\bar{\boldsymbol{\mu}}_-\ran\rvt>\frac{\rho}{2\sqrt{K}}\rp\leq 4\exp\lp -\frac{CD\rho^2}{4K\var^2}\rp\,,
    \end{align}
    }
    where \eqref{app_eq_ref_11} uses the fact that $ab+cd\leq \sqrt{a^2+c^2}\sqrt{b^2+d^2}$ for any $a,b,c,d\in\mathbb{R}$, a simple application of Cauchy-Schwarz inequality, and \eqref{app_eq_ref_12} uses the fact that $\bar{\boldsymbol{\mu}}_+,\bar{\boldsymbol{\mu}}_-$ are orthonormal, thus $\sqrt{\lvt\lan \boldsymbol{d},\bar{\boldsymbol{\mu}}_+\ran\rvt^2+\lvt\lan \boldsymbol{d},\bar{\boldsymbol{\mu}}_-\ran\rvt^2}\leq \|d\|\leq 1$.
    The last line uses Lemma \ref{lem_heoffding}. The proof of the case $k>K_1$ is identical to the one above by the symmetry of the problem.
    
    \textbf{Non-robustness of $F^{(p)}(\boldsymbol{x})$, complementary result} It suffices to show that
    \be
        \prob_{(\boldsymbol{x},y)\sim \mathcal{D}_{X,Y}}\lp \min_{\|\boldsymbol{d}\|\leq 1} F^{(p)}\lp \boldsymbol{x}+\frac{\sqrt{2}+\delta}{2}\boldsymbol{d}\rp y> 0 \mid z=k \rp \leq 4\exp\lp \frac{CD \delta^2}{8K^2\var^2}\rp\,, \forall k\leq K\,.
    \ee
    When $k\leq K_1$, we define $\boldsymbol{d}_k=\frac{-\boldsymbol{\mu}_k+\boldsymbol{\mu}_{K}}{\sqrt{2}}$, then
    \begin{align}
        &\prob_{(\boldsymbol{x},y)\sim \mathcal{D}_{X,Y}}\lp \min_{\|\boldsymbol{d}\|\leq 1} F^{(p)}\lp \boldsymbol{x}+\frac{\sqrt{2}+\delta}{2}\boldsymbol{d}\rp y> 0 \mid z=k \rp\nonumber\\
        \leq &\; \prob_{(\boldsymbol{x},y)\sim \mathcal{D}_{X,Y}}\lp  F^{(p)}\lp \boldsymbol{x}+\frac{\sqrt{2}+\delta}{2}\boldsymbol{d}_k\rp y> 0 \mid z=k \rp\nonumber\\
        = &\;  \prob_{\boldsymbol{\varepsilon}\sim \mathcal{N}\lp 0,\frac{\var^2}{D}\boldsymbol{I}_D\rp}\lp F^{(p)}\lp \boldsymbol{\mu}_k+\boldsymbol{\varepsilon}+\frac{\sqrt{2}+\delta}{2}\boldsymbol{d}_k\rp > 0\rp\nonumber\\
        \leq  &\;  \prob_{\boldsymbol{\varepsilon}\sim \mathcal{N}\lp 0,\frac{\var^2}{D}\boldsymbol{I}_D\rp}\lp 
        \act^p\lp \frac{1-\delta/\sqrt{2}}{2} + \lan \boldsymbol{\varepsilon}, \boldsymbol{\mu}_k\ran \rp - \act^p\lp\frac{1+\delta/\sqrt{2}}{2} + \lan \boldsymbol{\varepsilon}, \boldsymbol{\mu}_{K}\ran \rp + \sum_{l\neq k,l\neq K}|\lan \boldsymbol{\varepsilon}, \boldsymbol{\mu}_l\ran|^p>0
        \rp\nonumber\\
        \leq &\; \prob_{\boldsymbol{\varepsilon}\sim \mathcal{N}\lp 0,\frac{\var^2}{D}\boldsymbol{I}_D\rp}\lp 
        \act^p\lp \frac{1-\delta/\sqrt{2}}{2} + \lan \boldsymbol{\varepsilon}, \boldsymbol{\mu}_k\ran \rp  + \lp \sum_{l\neq k,l\neq K}|\lan \boldsymbol{\varepsilon}, \boldsymbol{\mu}_l\ran|\rp^p>\act^p\lp\frac{1+\delta/\sqrt{2}}{2}+ \lan \boldsymbol{\varepsilon}, \boldsymbol{\mu}_{K}\ran \rp
        \rp\nonumber\\
        \leq &\; \prob_{\boldsymbol{\varepsilon}\sim \mathcal{N}\lp 0,\frac{\var^2}{D}\boldsymbol{I}_D\rp}\lp 
        \act^p\lp \frac{1-\delta/\sqrt{2}}{2} + \lvt\lan \boldsymbol{\varepsilon}, \boldsymbol{\mu}_k\ran\rvt \rp  + \lp \sum_{l\neq k,l\neq K}|\lan \boldsymbol{\varepsilon}, \boldsymbol{\mu}_l\ran|\rp^p>\act^p\lp\frac{1+\delta/\sqrt{2}}{2}- \lvt\lan \boldsymbol{\varepsilon}, \boldsymbol{\mu}_{K}\ran\rvt \rp
        \rp\label{app_eq_ref_8}\\
        \leq &\; \prob_{\boldsymbol{\varepsilon}\sim \mathcal{N}\lp 0,\frac{\var^2}{D}\boldsymbol{I}_D\rp}\lp 
        \lp\act\lp \frac{1-\delta/\sqrt{2}}{2} + \lvt\lan \boldsymbol{\varepsilon}, \boldsymbol{\mu}_k\ran\rvt \rp  +  \sum_{l\neq k,l\neq K}|\lan \boldsymbol{\varepsilon}, \boldsymbol{\mu}_l\ran|\rp^p>\act^p\lp\frac{1+\delta/\sqrt{2}}{2}- \lvt\lan \boldsymbol{\varepsilon}, \boldsymbol{\mu}_{K}\ran\rvt \rp
        \rp\label{app_eq_ref_9}\\
        \leq &\; \prob_{\boldsymbol{\varepsilon}\sim \mathcal{N}\lp 0,\frac{\var^2}{D}\boldsymbol{I}_D\rp}\lp 
        \act\lp \frac{1-\delta/\sqrt{2}}{2} + \lvt\lan \boldsymbol{\varepsilon}, \boldsymbol{\mu}_k\ran\rvt \rp  +  \sum_{l\neq k,l\neq K}|\lan \boldsymbol{\varepsilon}, \boldsymbol{\mu}_l\ran|>\act\lp\frac{1+\delta/\sqrt{2}}{2}- \lvt\lan \boldsymbol{\varepsilon}, \boldsymbol{\mu}_{K}\ran\rvt \rp
        \rp\nonumber\\
        \leq &\; \prob_{\boldsymbol{\varepsilon}\sim \mathcal{N}\lp 0,\frac{\var^2}{D}\boldsymbol{I}_D\rp}\lp 
        \act\lp \frac{1-\delta/\sqrt{2}}{2} + \lvt\lan \boldsymbol{\varepsilon}, \boldsymbol{\mu}_k\ran\rvt \rp  +  \sum_{l\neq k,l\neq K}|\lan \boldsymbol{\varepsilon}, \boldsymbol{\mu}_l\ran|>\frac{1+\delta/\sqrt{2}}{2}- \lvt\lan \boldsymbol{\varepsilon}, \boldsymbol{\mu}_{K}\ran\rvt \rp\,,
        \label{eq_robust_5}
    \end{align}
    where \eqref{app_eq_ref_8} uses the fact that $\act^p(x)$ is non-decreasing w.r.t. $x$, and \eqref{app_eq_ref_9} uses the fact that $(a+b)^p\geq a^p+b^p$ for any $a,b>0$. We define the event 
    \be 
    \mathcal{E}_5:=\lb \frac{1-\delta/\sqrt{2}}{2} + \lan \boldsymbol{\varepsilon}, \boldsymbol{\mu}_k\ran>0\rb\,.\ee
    Then by Lemma \ref{lem_split_prob},
    \begin{align}
        \eqref{eq_robust_5}&\leq \prob_{\boldsymbol{\varepsilon}\sim \mathcal{N}\lp 0,\frac{\var^2}{D}\boldsymbol{I}_D\rp}\lp 
        \frac{1-\delta/\sqrt{2}}{2} + \lvt\lan \boldsymbol{\varepsilon}, \boldsymbol{\mu}_k\ran\rvt   +  \sum_{l\neq k,l\neq K}|\lan \boldsymbol{\varepsilon}, \boldsymbol{\mu}_l\ran|>\frac{1+\delta/\sqrt{2}}{2}- \lvt\lan \boldsymbol{\varepsilon}, \boldsymbol{\mu}_{K}\ran\rvt, \mathcal{E}_5
        \rp\nonumber\\
        &\qquad\qquad\qquad\qquad+ \prob_{\boldsymbol{\varepsilon}\sim \mathcal{N}\lp 0,\frac{\var^2}{D}\boldsymbol{I}_D\rp}\lp 
        \lvt\lan \boldsymbol{\varepsilon}, \boldsymbol{\mu}_k\ran\rvt   +  \sum_{l\neq k,l\neq K}|\lan \boldsymbol{\varepsilon}, \boldsymbol{\mu}_l\ran|>\frac{1+\delta/\sqrt{2}}{2}- \lvt\lan \boldsymbol{\varepsilon}, \boldsymbol{\mu}_{K}\ran\rvt, \mathcal{E}_5^c\rp\nonumber\\
        &\leq \prob_{\boldsymbol{\varepsilon}\sim \mathcal{N}\lp 0,\frac{\var^2}{D}\boldsymbol{I}_D\rp}\lp 
        \lvt\lan \boldsymbol{\varepsilon}, \boldsymbol{\mu}_{K}\ran\rvt+\lvt\lan \boldsymbol{\varepsilon}, \boldsymbol{\mu}_k\ran\rvt   +  \sum_{l\neq k,l\neq K}|\lan \boldsymbol{\varepsilon}, \boldsymbol{\mu}_l\ran|>\frac{\delta}{\sqrt{2}}
        \rp\nonumber\\
        &\qquad\qquad\qquad\qquad+ \prob_{\boldsymbol{\varepsilon}\sim \mathcal{N}\lp 0,\frac{\var^2}{D}\boldsymbol{I}_D\rp}\lp 
        \lvt\lan \boldsymbol{\varepsilon}, \boldsymbol{\mu}_{K}\ran\rvt+\lvt\lan \boldsymbol{\varepsilon}, \boldsymbol{\mu}_k\ran\rvt   +  \sum_{l\neq k,l\neq K}|\lan \boldsymbol{\varepsilon}, \boldsymbol{\mu}_l\ran|>\frac{1+\delta/\sqrt{2}}{2}
        \rp\nonumber\\
        &\leq 2\prob_{\boldsymbol{\varepsilon}\sim \mathcal{N}\lp 0,\frac{\var^2}{D}\boldsymbol{I}_D\rp}\lp 
        \sum_{1\leq l\leq K}|\lan \boldsymbol{\varepsilon}, \boldsymbol{\mu}_l\ran|>\frac{\delta}{2\sqrt{2}}
        \rp\nonumber\\
        &\leq 2\prob_{\boldsymbol{\varepsilon}\sim \mathcal{N}\lp 0,\frac{\var^2}{D}\boldsymbol{I}_D\rp}\lp 
        |\lan \boldsymbol{\varepsilon}, \boldsymbol{\mu}_1\ran|>\frac{\delta}{2\sqrt{2}K}\rp\leq 4\exp\lp \frac{CD \delta^2}{8K^2\var^2}
        \rp\,.
    \end{align}

    The last line uses Lemma \ref{lem_heoffding}. When $k>K_1$,  the proof is similar with $d_k:=\frac{-\boldsymbol{\mu}_k+\boldsymbol{\mu}_1}{\sqrt{2}}$.
\end{proof}

\newpage
\section{Proofs for Lemma \ref{lem_small_norm_informal} and Theorem \ref{thm_align}}\label{app_pf_align}
\subsection{Alignment bias illustrated}\label{app_align_bias_vis}
As we showed in Lemma \ref{lem_small_norm_informal}, during the alignment phase $t\leq T$, one have the following approximation
\be
    \frac{d}{dt} \frac{\boldsymbol{w}_j(t)}{\|\boldsymbol{w}_j(t)\|}\simeq \sign(v_j(0)) \mathcal{P}^\perp_{\boldsymbol{w}_j(t)} x^{(p)}(\boldsymbol{w}_j(t))\,,
\ee
with
\be
    x^{(p)}(w)=\sum_{k=1}^K\gamma_k(w)y_k\boldsymbol{x}_k\cdot p[\cos(\boldsymbol{x}_k,w)]^{p-1}\,,
\ee
which essentially shows that when $\boldsymbol{w}_j$ is a \emph{positive neuron} ($\sign(v_j(0))>0$), then gradient flow dynamics during alignment phase pushes $\boldsymbol{w}_j/\|\boldsymbol{w}_j\|$ toward the direction of $x^{(p)}(w)$.

Notably, $x^{(p)}(\boldsymbol{w}_j)$ critically depends on $p$. Roughly speaking, when $p=1$, $x^{(p)}(\boldsymbol{w}_j)$ are more aligned with $\bar{\boldsymbol{\mu}}_+$ and $\bar{\boldsymbol{\mu}}_-$, while when $p>3$, $x^{(p)}(\boldsymbol{w}_j)$ are more aligned with one of the subclass centers, thus by moving toward $x^{(p)}(\boldsymbol{w}_j)$ in direction, the neurons are likely to align with average class centers in the former case, and with subclass centers in the latter case. We elaborate this statement here with a toy example.

\begin{figure*}[!h]
  \centering
  \includegraphics[width=0.5\linewidth]{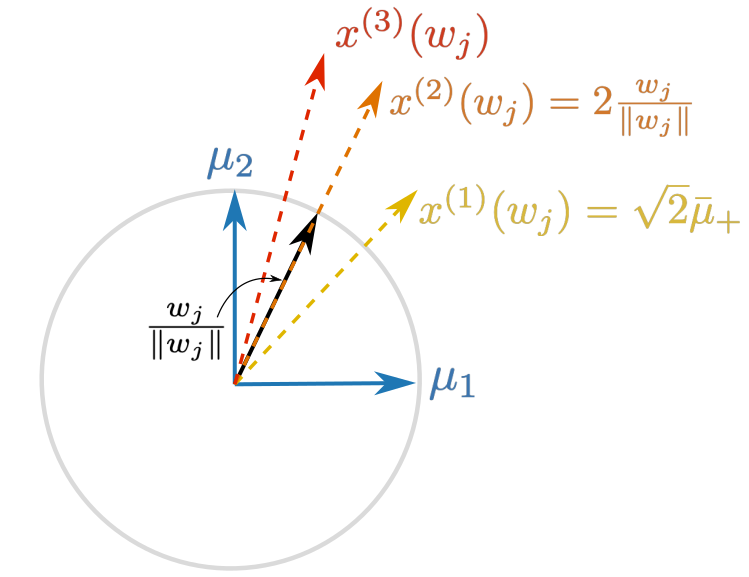}
  \caption{Alignment bias visualized. During alignment phase $\boldsymbol{w}_j$ is moving toward $x^{(p)}(\boldsymbol{w}_j)$ in direction. When $p=1$, $x^{(1)}(\boldsymbol{w}_j)$ is aligned with average class center $\bar{\boldsymbol{\mu}}_+$; When $p=3$, $x^{(p)}(\boldsymbol{w}_j)$ is more aligned with one of the subclass centers $\boldsymbol{\mu}_1$ and $\boldsymbol{\mu}_2$, depending on which one is closer to $\boldsymbol{w}_j$ in cosine distance.}
  \label{fig_align_bias}
\end{figure*}

Suppose the dataset ($K=2, K_1=2$) only contains two orthogonal $\boldsymbol{\mu}_1$, and $\boldsymbol{\mu}_2$ in the $2$-d plane and they both have positive labels. Given a positive neuron $\boldsymbol{w}_j$ that is activated by both $\boldsymbol{\mu}_1$, and $\boldsymbol{\mu}_2$, as shown in Figure \ref{fig_align_bias}. During alignment phase $\frac{\boldsymbol{w}_j}{\|\boldsymbol{w}_j\|}$ is moving towards the direction of $x^{(p)}(\boldsymbol{w}_j)$, which is
\begin{itemize}
    \item when $p=1$,
    \be
        x^{(1)}(\boldsymbol{w}_j)=\sum_{k=1}^K\gamma_k(\boldsymbol{w}_j)y_k\boldsymbol{x}_k=\boldsymbol{\mu}_1+\boldsymbol{\mu}_2=\sqrt{2}\bar{\boldsymbol{\mu}}_+\,,
    \ee
    exactly aligned with average class center $\bar{\boldsymbol{\mu}}_+$.
    \item when $p=2$,
    \be
        x^{(2)}(\boldsymbol{w}_j)=\sum_{k=1}^K\gamma_k(\boldsymbol{w}_j)y_k\boldsymbol{x}_k\cdot 2[\cos(\boldsymbol{x}_k,w)]=2\lp\boldsymbol{\mu}_1\cos\lp \boldsymbol{\mu}_1,\boldsymbol{w}_j\rp+\boldsymbol{\mu}_2\cos\lp \boldsymbol{\mu}_2,\boldsymbol{w}_j\rp\rp=2\frac{\boldsymbol{w}_j}{\|\boldsymbol{w}_j\|}\,,
    \ee
    exactly aligned with $\boldsymbol{w}_j$ itself.
    \item when $p=3$,
    \be
        x^{(3)}(w)=\sum_{k=1}^K\gamma_k(\boldsymbol{w}_j)y_k\boldsymbol{x}_k\cdot 3[\cos(\boldsymbol{x}_k,\boldsymbol{w}_j)]^2=3\lp\boldsymbol{\mu}_1\cos\lp \boldsymbol{\mu}_1,\boldsymbol{w}_j\rp^2+\boldsymbol{\mu}_2\cos\lp \boldsymbol{\mu}_2,\boldsymbol{w}_j\rp^2\rp\,,
    \ee
    getting closer to either $\boldsymbol{\mu}_1$ or $\boldsymbol{\mu}_2$, depending which one is closer to $\boldsymbol{w}_j$ in cosine distance.
\end{itemize}
Although this example is even more simplified than the one in Section \ref{sec:theoretical_analsysis}, it is easy to visualize and keeps the core relationship between the alignment bias of the neurons and the pReLU activation. From this, we see how the alignment bias is altered under different choices of $p$.

\subsection{Auxiliary Lemma}
We first prove the following, most analyses on gradient flow with small initialization~\cite{boursier2022gradient,boursier2024early,min2023early} have similar results, saying that the norm of the neurons stays close to zero during the alignment phase.
\begin{lemma}\label{lem_small_norm}
    Given some initialization in \eqref{eq_init}, then for any $\epsilon\leq \frac{1}{4\sqrt{h} M^2}$, any solution to the gradient flow dynamics under the simplified training dataset satisfies
    \be
        \max_j\|\boldsymbol{w}_j(t)\|^2\leq \frac{2\epsilon M^2}{\sqrt{h}},\quad \max|f_p(\boldsymbol{x}_k;\boldsymbol{\theta}(t))|\leq 2\epsilon \sqrt{h}M^2\,,
    \ee
    $\forall t\leq \frac{1}{4K}\log\frac{1}{\sqrt{h}\epsilon}
    $.
\end{lemma}
and we need the following lemma
\begin{lemma}\label{lem_sc}
    Given nonnegative $z_1,\cdots,z_n$, consider a function
    \be
        g_p(q;\{z_i\}_{i=1}^n)=\lp\sum_{i=1}^nz_i^q\rp\lp\sum_{i=1}^nz_i^{p+1-q}\rp\,,
    \ee
    then $g_p(q;\{z_i\}_{i=1}^n)$ is convex on $\mathbb{R}$. Moreover, as long as $z_i\neq z_j$ for some $i,j$, then $g_p(q;\{z_i\}_{i=1}^n)$ is strictly convex with minimum at $q^*=\frac{p+1}{2}$.
\end{lemma}
we leave their proofs at the end of this section. Lastly, we use the following lemma from~\citet{min2023early}
\begin{lemma}\label{assump_loss}
    For $\ell$ being either exponential or logistic loss, we have
    \be
        |-\nabla_{\hat{y}}\ell (y,\hat{y})-y|\leq 2|\hat{y}|,\forall y\in\{+1,-1\},\quad \forall |\hat{y}|\leq 1\,.
    \ee
\end{lemma}
\subsection{Proof for Lemma \ref{lem_small_norm_informal}}
\begin{customlem}{1}[restated]
    Given some initialization from \eqref{eq_init}, if $\epsilon=\mathcal{O}( \frac{1}{\sqrt{h}})$, then there exists $T=\boldsymbol{\theta}(\frac{1}{K}\log\frac{1}{\sqrt{h}\epsilon})$ such that the trajectory under gradient flow training with the simplified training dataset almost surely satisfies that $\forall t\leq T$,
\begin{align*}
    &\;\max_j\lV \frac{d}{dt} \frac{\boldsymbol{w}_j(t)}{\|\boldsymbol{w}_j(t)\|}-\sign(v_j(0)) \mathcal{P}^\perp_{\boldsymbol{w}_j(t)} x^{(p)}(\boldsymbol{w}_j(t))\rV=\mathcal{O}\lp\epsilon k\sqrt{h}\rp\,,\label{eq_err}
\end{align*}
where $\mathcal{P}^\perp_{w}=I-\frac{ww^\top}{\|w\|^2}$ and
\be
    x^{(p)}(w)=\sum_{k=1}^K\gamma_k(w)y_k\boldsymbol{x}_k\cdot p[\cos(\boldsymbol{x}_k,w)]^{p-1},
\ee
with $\gamma_k(w)$ being a subgradient of $\act^p(z)$ at $z=\lan \boldsymbol{x}_k,w\ran$.
\end{customlem}
\begin{proof}
    For simplicity, we write $\boldsymbol{w}_j(t)$ as $\boldsymbol{w}_j$.
    
    As we will show in the proof for Lemma \ref{lem_small_norm}, under balanced initialization, 
    \be
        \frac{d}{dt}\boldsymbol{w}_j=-\sum_{k=1}^K\gamma_k(\boldsymbol{w}_j)\nabla_{\hat{y}}\ell(y_k,f_p(\boldsymbol{x}_k;\boldsymbol{\theta})) \|\boldsymbol{w}_j\|\lp \frac{p[\act(\lan \boldsymbol{x}_k,\boldsymbol{w}_j\ran)]^{p-1}}{\|\boldsymbol{w}_j\|^{p-1}}\boldsymbol{x}_k-(p-1)\frac{[\act(\lan \boldsymbol{x}_k,\boldsymbol{w}_j\ran)]^p}{\|\boldsymbol{w}_j\|^{p+1}}\boldsymbol{w}_j\rp\,.
    \ee
    Then for any $j\in[h]$,
    \begin{align*}
        \frac{d}{dt} \frac{\boldsymbol{w}_j}{\|\boldsymbol{w}_j\|}&=\; \mathcal{P}^\perp_{\boldsymbol{w}_j}\cdot \frac{1}{\|\boldsymbol{w}_j\|}\cdot\frac{d}{dt}\boldsymbol{w}_j\\
        &=-\sign(v_j(0))\sum_{k=1}^K\gamma_k(\boldsymbol{w}_j)\nabla_{\hat{y}}\ell(y_k,f_p(\boldsymbol{x}_k;\boldsymbol{\theta}))\mathcal{P}^\perp_{\boldsymbol{w}_j}\lp p\cos\lp \boldsymbol{x}_k,\boldsymbol{w}_j\rp^{p-1}\boldsymbol{x}_k-(p-1)\frac{[\act(\lan \boldsymbol{x}_k,\boldsymbol{w}_j\ran)]^p}{\|\boldsymbol{w}_j\|^{p+1}}\boldsymbol{w}_j\rp\\
        &=\; -\sign(v_j(0))\sum_{i=1}^K\gamma_k(\boldsymbol{w}_j)\nabla_{\hat{y}}\ell(y_k,f_p(\boldsymbol{x}_k;\boldsymbol{\theta}))\mathcal{P}^\perp_{\boldsymbol{w}_j}p\cos\lp \boldsymbol{x}_k,\boldsymbol{w}_j\rp^{p-1}\boldsymbol{x}_k\,,
    \end{align*}
    Then
    \begin{align*}
    &\;\max_j\lV \frac{d}{dt} \frac{\boldsymbol{w}_j(t)}{\|\boldsymbol{w}_j(t)\|}-\sign(v_j(0)) \mathcal{P}^\perp_{\boldsymbol{w}_j(t)} x^{(p)}(\boldsymbol{w}_j(t))\rV\\
    =&\; \max_j\lV \sum_{i=1}^K\gamma_k(\boldsymbol{w}_j)\nabla_{\hat{y}}(-\ell(y_k,f_p(\boldsymbol{x}_k;\boldsymbol{\theta}))-y_k)\mathcal{P}^\perp_{\boldsymbol{w}_j}p\cos\lp \boldsymbol{x}_k,\boldsymbol{w}_j\rp^{p-1}\boldsymbol{x}_k \rV\\
    &\leq \; \max_j\lV \sum_{i=1}^K\gamma_k(\boldsymbol{w}_j)\lvt \nabla_{\hat{y}}(-\ell(y_k,f_p(\boldsymbol{x}_k;\boldsymbol{\theta}))-y_k\rvt \mathcal{P}^\perp_{\boldsymbol{w}_j}p\cos\lp \boldsymbol{x}_k,\boldsymbol{w}_j\rp^{p-1}\boldsymbol{x}_k \rV\leq 2Kp\max_k|f_p(\boldsymbol{x}_k;\boldsymbol{\theta}(t))|\,,
    \end{align*}
    by Lemma \ref{assump_loss}. Finally, by Lemma \ref{lem_small_norm}, we have for any $\epsilon\leq \frac{1}{4\sqrt{h}M^2}$, $\forall t\leq T=\frac{1}{4K}\log\frac{1}{\sqrt{h}\epsilon}$, we have
    \be
        \max_j\lV \frac{d}{dt} \frac{\boldsymbol{w}_j(t)}{\|\boldsymbol{w}_j(t)\|}-\sign(v_j(0)) \mathcal{P}^\perp_{\boldsymbol{w}_j(t)} x^{(p)}(\boldsymbol{w}_j(t))\rV\leq 2Kp\max|f_p(\boldsymbol{x}_k;\boldsymbol{\theta}(t))|\leq 4\epsilon \sqrt{h}M^2Kp\,,
    \ee
    which finishes the proof.
\end{proof}
\subsection{Proof for Theorem \ref{thm_align}}
\begin{customthm}{2}[Alignment bias of neurons, complete statement]\label{thm_align_complete}
    Given some $0<\underline{\delta}<\delta<1$ and a fixed choice of $p\geq 1$, then $\exists \epsilon_0:=\epsilon_0(\delta,p)>0$ such that for any solution of the gradient flow on $f_p(x;\boldsymbol{\theta})$ with the simplified training dataset, starting from some initialization from \eqref{eq_init} with initialization scale $\epsilon<\epsilon_0$, almost surely we have that at any time $t\leq T=\boldsymbol{\theta}(\frac{1}{n}\log\frac{1}{\sqrt{h}\epsilon})$ and 
    
    \begin{itemize}[leftmargin=9pt,topsep=0pt,parsep=0pt]
        \item $\forall j$ with $\sign(v_j(0))>0$, 
    \be
        \left.\frac{d}{dt}\! \cos\!\lp \boldsymbol{w}_j(t), \bar{\boldsymbol{\mu}}_+\rp \rvt_{\cos(\boldsymbol{w}_j(t),\bar{\boldsymbol{\mu}}_+)=1-\delta}\! \begin{cases}
            >0,&\! \text{when } p=1\\
            <0,&\! \text{when } p\geq 3\\
        \end{cases},
    \ee
    and
    \be
        \left.\frac{d}{dt}\! \cos\!\lp \boldsymbol{w}_j(t), \boldsymbol{\mu}_k\rp \rvt_{\cos(\boldsymbol{w}_j(t),\boldsymbol{\mu}_k)=1-\delta}>0,\forall k\leq K_1\,.
    \ee
    \item $\forall j$ with $\sign(v_j(0))<0$, 
    \be
        \left.\frac{d}{dt}\! \cos\!\lp \boldsymbol{w}_j(t), \bar{\boldsymbol{\mu}}_-\rp \rvt_{\cos(\boldsymbol{w}_j(t),\bar{\boldsymbol{\mu}}_-)=1-\delta}\! \begin{cases}
            >0,&\! \text{when } p=1\\
            <0,&\! \text{when } p\geq 3\\
        \end{cases},
    \ee
    and
    \be
        \left.\frac{d}{dt}\! \cos\!\lp \boldsymbol{w}_j(t), \boldsymbol{\mu}_k\rp \rvt_{\cos(\boldsymbol{w}_j(t),\boldsymbol{\mu}_k)=1-\delta}>0,\forall k> K_1\,.
    \ee
    \end{itemize}
\end{customthm}
\begin{proof}
    The proofs for positive neurons and for negative neurons are almost identical, we will prove it for positive neurons $\sign(v_j(0))>0$, i.e. $\sign(v_j(0))=1$. The first part concerns about $\left.\frac{d}{dt} \cos\lp \boldsymbol{w}_j(t), \bar{\boldsymbol{\mu}}_+\rp \rvt_{\cos(\boldsymbol{w}_j(t),\bar{\boldsymbol{\mu}}_+)=1-\delta}$. 
    \begin{align}
        &\frac{d}{dt} \cos\lp \boldsymbol{w}_j(t), \bar{\boldsymbol{\mu}}_+\rp\\
        =&\; \frac{1}{\|\bar{\boldsymbol{\mu}}_+\|}\lan \frac{d}{dt}\frac{\boldsymbol{w}_j(t)}{\|\boldsymbol{w}_j(t)\|},\bar{\boldsymbol{\mu}}_+\ran\\
        =&\; \frac{1}{\|\bar{\boldsymbol{\mu}}_+\|}\lan \mathcal{P}^\perp_{\boldsymbol{w}_j(t)} x^{(p)}(\boldsymbol{w}_j(t)),\bar{\boldsymbol{\mu}}_+\ran + \frac{1}{\|\bar{\boldsymbol{\mu}}_+\|}\lan\frac{d}{dt} \frac{\boldsymbol{w}_j(t)}{\|\boldsymbol{w}_j(t)\|}-\sign(v_j(0)) \mathcal{P}^\perp_{\boldsymbol{w}_j(t)} x^{(p)}(\boldsymbol{w}_j(t)),\bar{\boldsymbol{\mu}}_+\ran\,,
    \end{align}
    \textbf{When $p=1$}, if we can show that for some choice of $\delta>0$,
    \be
        \inf_{\boldsymbol{w}_j\in\mathbb{S}^{D-1}, \cos\lp \boldsymbol{w}_j,\bar{\boldsymbol{\mu}}_+\rp=1-\delta}\frac{1}{\|\bar{\boldsymbol{\mu}}_+\|}\lan \mathcal{P}^\perp_{\boldsymbol{w}_j} x^{(1)}(\boldsymbol{w}_j),\bar{\boldsymbol{\mu}}_+\ran := \Delta_1(\delta) >0\,, 
    \ee
    then we pick $\epsilon\leq \epsilon_0=\frac{\Delta_1(\delta)}{8\sqrt{h}M^2Kp}$, and by Lemma \ref{lem_small_norm_informal}, we have that for $\forall t\leq T=\frac{1}{4K}\log\frac{1}{\sqrt{h}\epsilon}$,
    \begin{align}
        \left.\frac{d}{dt} \cos\lp \boldsymbol{w}_j(t), \bar{\boldsymbol{\mu}}_+\rp \rvt_{\cos(\boldsymbol{w}_j(t),\bar{\boldsymbol{\mu}}_+)=1-\delta}
        &\geq\;  \Delta_1(\delta) - \max_j\lV \frac{d}{dt} \frac{\boldsymbol{w}_j(t)}{\|\boldsymbol{w}_j(t)\|}-\sign(v_j(0)) \mathcal{P}^\perp_{\boldsymbol{w}_j(t)} x^{(p)}(\boldsymbol{w}_j(t))\rV\\
        &\geq\; \Delta_1(\delta) - 4\epsilon \sqrt{h}M^2Kp \geq \frac{\Delta_1(\delta)}{2}>0\,,
    \end{align}
    which is what we stated in the theorem. 
    
    Similarly, \textbf{When $p>3$},  if we can show that for some choice of $\delta>0$,
    \be
        \inf_{\boldsymbol{w}_j\in\mathbb{S}^{D-1}, \cos\lp \boldsymbol{w}_j,\bar{\boldsymbol{\mu}}_+\rp=1-\delta}\frac{1}{\|\bar{\boldsymbol{\mu}}_+\|}\lan \mathcal{P}^\perp_{\boldsymbol{w}_j} x^{(p)}(\boldsymbol{w}_j),\bar{\boldsymbol{\mu}}_+\ran := \Delta_p(\delta) <0\,, 
    \ee
    then we pick $\epsilon\leq \epsilon_0=\frac{\Delta_p(\delta)}{8\sqrt{h}M^2Kp}$, and by Lemma \ref{lem_small_norm_informal}, we have that for $\forall t\leq T=\frac{1}{4K}\log\frac{1}{\sqrt{h}\epsilon}$,
    \begin{align}
        \left.\frac{d}{dt} \cos\lp \boldsymbol{w}_j(t), \bar{\boldsymbol{\mu}}_+\rp \rvt_{\cos(\boldsymbol{w}_j(t),\bar{\boldsymbol{\mu}}_+)=1-\delta}
        &\leq\;  \Delta_p(\delta) + \max_j\lV \frac{d}{dt} \frac{\boldsymbol{w}_j(t)}{\|\boldsymbol{w}_j(t)\|}-\sign(v_j(0)) \mathcal{P}^\perp_{\boldsymbol{w}_j(t)} x^{(p)}(\boldsymbol{w}_j(t))\rV\\
        &\leq\; \Delta_p(\delta) + 4\epsilon \sqrt{h}M^2Kp \leq \frac{\Delta_p(\delta)}{2}<0\,.
    \end{align}
    Therefore, for the first part, it suffices to show 
    \be
        \inf_{\boldsymbol{w}_j\in\mathbb{S}^{D-1}, \cos\lp \boldsymbol{w}_j,\bar{\boldsymbol{\mu}}_+\rp=1-\delta}\frac{1}{\|\bar{\boldsymbol{\mu}}_+\|}\lan \mathcal{P}^\perp_{\boldsymbol{w}_j} x^{(p)}(\boldsymbol{w}_j),\bar{\boldsymbol{\mu}}_+\ran := \Delta_p(\delta)\begin{cases}
            >0, & \text{for } p=1\\
            >0, & \text{for } p\geq 3\\
        \end{cases}
    \ee
    Now there exists $1>\bar{\delta}_1>0$ such that when $\delta>\bar{\delta}_1$, and $\cos\lp \boldsymbol{w}_j,\bar{\boldsymbol{\mu}}_+\rp=\sqrt{1-\delta}$, we have $\gamma_k(\boldsymbol{w}_j)=1,\forall k\leq K_1$ and $\gamma_k(\boldsymbol{w}_j)=0,\forall k> K_1$, i.e., $\boldsymbol{w}_j$ is activated by all $\boldsymbol{x}_k,k\geq K_1$ with positive label and is not activated by any of the $\boldsymbol{x}_k,k> K_1$ with negative label. Moreover, there exists $z_k,k\leq K_1$, such that $\boldsymbol{w}_j=\|\boldsymbol{w}_j\|\sum_{k\leq K_1}z_k\boldsymbol{x}_k$ and $z_k=\lan \boldsymbol{x}_k, \frac{\boldsymbol{w}_j}{\|\boldsymbol{w}_j\|}\ran$, i.e., $\boldsymbol{w}_j$ lies completely within the span of $\boldsymbol{x}_k,k\leq K_1$.
    
    With this, we have
    \begin{align*}
        &\frac{1}{\|\bar{\boldsymbol{\mu}}_+\|}\lan \mathcal{P}^\perp_{\boldsymbol{w}_j} x^{(p)}(\boldsymbol{w}_j),\bar{\boldsymbol{\mu}}_+\ran\\
        =&\; \frac{1}{\|\bar{\boldsymbol{\mu}}_+\|}\lan \lp I-\frac{\boldsymbol{w}_j\boldsymbol{w}_j^T}{\|\boldsymbol{w}_j\|^2}\rp\sum_{k}^K\gamma_k(\boldsymbol{w}_j)y_k\boldsymbol{x}_k\cdot p[\cos(\boldsymbol{x}_k,\boldsymbol{w}_j)]^{p-1},\bar{\boldsymbol{\mu}}_+\ran \\
        =&\; \frac{1}{K}\lan \lp I-\frac{\boldsymbol{w}_j\boldsymbol{w}_j^T}{\|\boldsymbol{w}_j\|^2}\rp\sum_{k\leq K_1}\boldsymbol{x}_k\cdot p[\cos(\boldsymbol{x}_k,\boldsymbol{w}_j)]^{p-1}, \sum_{i\leq K_1}\boldsymbol{x}_k\ran\\
        =&\; \lan \sum_{k\leq K_1}\boldsymbol{x}_k, \sum_{k\leq K_1}\boldsymbol{x}_k\cdot p[\cos(\boldsymbol{x}_k,\boldsymbol{w}_j)]^{p-1}\ran - \lan \sum_{k\leq K_1}\boldsymbol{x}_k, \frac{\boldsymbol{w}_j}{\|\boldsymbol{w}_j\|} \ran \sum_{k\leq K_1}\lan \boldsymbol{x}_k, \frac{\boldsymbol{w}_j}{\|\boldsymbol{w}_j\|}\ran \cdot p[\cos(\boldsymbol{x}_k,\boldsymbol{w}_j)]^{p-1}\\
        =&\; \lan \sum_{k\leq K_1}\boldsymbol{x}_k, \sum_{k\leq K_1}\boldsymbol{x}_k\cdot p \lan \boldsymbol{x}_k, \frac{\boldsymbol{w}_j}{\|\boldsymbol{w}_j\|}\ran^{p-1}\ran - \lan \sum_{k\leq K_1}\boldsymbol{x}_k, \frac{\boldsymbol{w}_j}{\|\boldsymbol{w}_j\|} \ran \sum_{k\leq K_1}\lan \boldsymbol{x}_k, \frac{\boldsymbol{w}_j}{\|\boldsymbol{w}_j\|}\ran \cdot p\lan \boldsymbol{x}_k, \frac{\boldsymbol{w}_j}{\|\boldsymbol{w}_j\|}\ran^{p-1}\\
        =&\;\sum_{k\leq K_1} p \lan \boldsymbol{x}_k, \frac{\boldsymbol{w}_j}{\|\boldsymbol{w}_j\|}\ran^{p-1} - \lp\sum_{k\leq K_1}\lan \boldsymbol{x}_k, \frac{\boldsymbol{w}_j}{\|\boldsymbol{w}_j\|}\ran\rp \lp p\sum_{k\leq K_1}\lan \boldsymbol{x}_k, \frac{\boldsymbol{w}_j}{\|\boldsymbol{w}_j\|}\ran^{p}\rp\\
        =&\;p\sum_{k\leq K_1} \lan \boldsymbol{x}_k, \frac{\boldsymbol{w}_j}{\|\boldsymbol{w}_j\|}\ran^{p-1} - p\lp\sum_{k\leq K_1}\lan \boldsymbol{x}_k, \frac{\boldsymbol{w}_j}{\|\boldsymbol{w}_j\|}\ran\rp \lp \sum_{k\leq K_1}\lan \boldsymbol{x}_k, \frac{\boldsymbol{w}_j}{\|\boldsymbol{w}_j\|}\ran^{p}\rp\\
        =&\;p \lp\sum_{k\leq K_1}z_k^{p-1}-\lp\sum_{k\leq K_1}z_k\rp\lp\sum_{k\leq K_1}z_k^{p}\rp\rp\,,
    \end{align*}
    Since $\boldsymbol{w}_j$ lies completely within the span of $\boldsymbol{x}_k,k\leq K_1$, we have $\sum_{k\leq K_1}z_k^2=1$, then
    \begin{align*}
        &p \lp\sum_{k\leq K_1}z_k^{p-1}-\lp\sum_{k\leq K_1}z_k\rp\lp\sum_{k\leq K_1}z_k^{p}\rp\rp\\
        =& p \lp\lp\sum_{k\leq K_1}z_k^{p-1}\rp\lp\sum_{k\leq K_1}z_k^2\rp-\lp\sum_{k\leq K_1}z_k\rp\lp\sum_{k\leq K_1}z_k^{p}\rp\rp=p\lhp g_p(2;\{z_k\}_{k\leq K_1})-g_p(1;\{z_k\}_{k\leq K_1})\rhp\,,
    \end{align*}
    where $g_p(\cdot;\{z_k\}_{k\leq K_1})$ is defined in Lemma \ref{lem_sc}. 
    
    By Lemma \ref{lem_sc}, when $p=1$, $g_1(\cdot;\{z_k\}_{k\leq K_1})$ is strictly convex and takes minimum at $q^*=\frac{1+p}{2}=1$, thus 
    \be
        g_1(2;\{z_k\}_{k\leq K_1})-g_1(1;\{z_k\}_{k\leq K_1})>0\,,
    \ee
    then we know that 
    \be
        \inf_{\boldsymbol{w}_j\in\mathbb{S}^{D-1}, \cos\lp \boldsymbol{w}_j,\bar{\boldsymbol{\mu}}_+\rp=1-\delta}\frac{1}{\|\bar{\boldsymbol{\mu}}_+\|}\lan \mathcal{P}^\perp_{\boldsymbol{w}_j} x^{(1)}(\boldsymbol{w}_j),\bar{\boldsymbol{\mu}}_+\ran := \Delta_1(\delta)\geq 0\,.
    \ee
    However, $\Delta_1(\delta)$ can not be zero: If this is the case, since the set $\{\boldsymbol{w}_j: \boldsymbol{w}_j\in\mathbb{S}^{D-1}, \cos\lp \boldsymbol{w}_j,\bar{\boldsymbol{\mu}}_+\rp=1-\delta\}$ is compact and $\frac{1}{\|\bar{\boldsymbol{\mu}}_+\|}\lan \mathcal{P}^\perp_{\boldsymbol{w}_j} x^{(1)}(\boldsymbol{w}_j),\bar{\boldsymbol{\mu}}_+\ran$ is continuous on this set. It attains minimum $0$ at some $\boldsymbol{w}_j$, which implies the non-strong convexity of $g_1(\cdot;\{z_k\}_{k\leq K_1})$ that, by Lemma \ref{lem_sc}, requires all $z_k,k\leq K_1$ to be equal to each other (This can only happen if $\cos\lp \boldsymbol{w}_j,\bar{\boldsymbol{\mu}}_+\rp=1$ ). Contradiction. Then one must have
    \be
        \inf_{\boldsymbol{w}_j\in\mathbb{S}^{D-1}, \cos\lp \boldsymbol{w}_j,\bar{\boldsymbol{\mu}}_+\rp=1-\delta}\frac{1}{\|\bar{\boldsymbol{\mu}}_+\|}\lan \mathcal{P}^\perp_{\boldsymbol{w}_j} x^{(1)}(\boldsymbol{w}_j),\bar{\boldsymbol{\mu}}_+\ran := \Delta_1(\delta)>0 \,.
    \ee

    Similarly, by Lemma \ref{lem_sc}, when $p\geq 3$, $g_1(\cdot;\{z_k\}_{k\leq K_1})$ is strictly convex and takes minimum at $q^*=\frac{1+p}{2}\geq 2$, thus 
    \be
        g_1(2;\{z_k\}_{k\leq K_1})-g_1(1;\{z_k\}_{k\leq K_1})<0\,,
    \ee
    then we know that 
    \be
        \inf_{\boldsymbol{w}_j\in\mathbb{S}^{D-1}, \cos\lp \boldsymbol{w}_j,\bar{\boldsymbol{\mu}}_+\rp=1-\delta}\frac{1}{\|\bar{\boldsymbol{\mu}}_+\|}\lan \mathcal{P}^\perp_{\boldsymbol{w}_j} x^{(1)}(\boldsymbol{w}_j),\bar{\boldsymbol{\mu}}_+\ran := \Delta_p(\delta)\leq 0\,.
    \ee
    Using the same argument, we eliminate the case of this infimum being zero. Then one must have
    \be
        \inf_{\boldsymbol{w}_j\in\mathbb{S}^{D-1}, \cos\lp \boldsymbol{w}_j,\bar{\boldsymbol{\mu}}_+\rp=1-\delta}\frac{1}{\|\bar{\boldsymbol{\mu}}_+\|}\lan \mathcal{P}^\perp_{\boldsymbol{w}_j} x^{(1)}(\boldsymbol{w}_j),\bar{\boldsymbol{\mu}}_+\ran := \Delta_p(\delta)<0 \,.
    \ee

    The second part concerns about $\left.\frac{d}{dt} \cos\lp \boldsymbol{w}_j(t), \boldsymbol{\mu}_k\rp \rvt_{\cos(\boldsymbol{w}_j(t),\boldsymbol{\mu}_k)=1-\delta}$, for some $k\leq K_1$. Without loss of generality, we let $k=1$. Thus we intend to show $\left.\frac{d}{dt} \cos\lp \boldsymbol{w}_j(t), \boldsymbol{\mu}_1\rp \rvt_{\cos(\boldsymbol{w}_j(t),\boldsymbol{\mu}_1)=1-\delta}$ is positive.
    
    We also let $\zeta:=1-(1-\delta)^2$, so that the condition $\cos(\boldsymbol{w}_j(t),\boldsymbol{\mu}_1)=1-\delta$ becomes $\cos(\boldsymbol{w}_j(t),\boldsymbol{\mu}_1)=\sqrt{1-\zeta}$. 

    Since $\sum_{k=1}^K\lvt\lan \frac{\boldsymbol{w}_j}{\|\boldsymbol{w}_j\|},\boldsymbol{\mu}_k\ran\rvt^2\leq 1$,  $\cos(\boldsymbol{w}_j(t),\boldsymbol{\mu}_1)=\sqrt{1-\zeta}$ implies $\sum_{l=2}^K\lvt\lan \frac{\boldsymbol{w}_j}{\|\boldsymbol{w}_j\|},\boldsymbol{\mu}_l\ran\rvt^2\leq \zeta$.

    Similar to the first part of the proof, we have 
    \begin{align}
        &\frac{d}{dt} \cos\lp \boldsymbol{w}_j(t), \boldsymbol{\mu}_1\rp\\
        =&\; \lan \frac{d}{dt}\frac{\boldsymbol{w}_j(t)}{\|\boldsymbol{w}_j(t)\|},\boldsymbol{\mu}_1\ran\\
        =&\; \lan \mathcal{P}^\perp_{\boldsymbol{w}_j(t)} x^{(p)}(\boldsymbol{w}_j(t)),\boldsymbol{\mu}_1\ran + \lan\frac{d}{dt} \frac{\boldsymbol{w}_j(t)}{\|\boldsymbol{w}_j(t)\|}-\sign(v_j(0)) \mathcal{P}^\perp_{\boldsymbol{w}_j(t)} x^{(p)}(\boldsymbol{w}_j(t)),\boldsymbol{\mu}_1\ran\,,
    \end{align}
    if we can show that for some choice of $\delta>0$ (or equivalently some $\zeta>0$),
    \be
        \inf_{\boldsymbol{w}_j\in\mathbb{S}^{D-1}, \cos\lp \boldsymbol{w}_j,\boldsymbol{\mu}_1\rp=1-\delta}\lan \mathcal{P}^\perp_{\boldsymbol{w}_j} x^{(1)}(\boldsymbol{w}_j),\boldsymbol{\mu}_1\ran := \Lambda_p(\delta) >0\,, 
    \ee
    then we pick $\epsilon\leq \epsilon_0=\frac{\Lambda_p(\delta)}{8\sqrt{h}M^2Kp}$, and by Lemma \ref{lem_small_norm_informal}, we have that for $\forall t\leq T=\frac{1}{4K}\log\frac{1}{\sqrt{h}\epsilon}$,
    \begin{align}
        \left.\frac{d}{dt} \cos\lp \boldsymbol{w}_j(t), \boldsymbol{\mu}_1\rp \rvt_{\cos(\boldsymbol{w}_j(t),\boldsymbol{\mu}_1)=1-\delta}
        &\geq\;  \Lambda_p(\delta) - \max_j\lV \frac{d}{dt} \frac{\boldsymbol{w}_j(t)}{\|\boldsymbol{w}_j(t)\|}-\sign(v_j(0)) \mathcal{P}^\perp_{\boldsymbol{w}_j(t)} x^{(p)}(\boldsymbol{w}_j(t))\rV\\
        &\geq\; \Lambda_p(\delta) - 4\epsilon \sqrt{h}M^2Kp \geq \frac{\Lambda_p(\delta)}{2}>0\,,
    \end{align}
    which is what we stated in the theorem. The remaining proof is to find a lower bound on $\Lambda_p(\delta)$, which is given by
    \begin{align*}
        &\lan \mathcal{P}^\perp_{\boldsymbol{w}_j} x^{(p)}(\boldsymbol{w}_j),\boldsymbol{\mu}_1\ran\\
        =&\; \lan \lp I-\frac{\boldsymbol{w}_j\boldsymbol{w}_j^T}{\|\boldsymbol{w}_j\|^2}\rp\sum_{k}^K\gamma_k(\boldsymbol{w}_j)y_k\boldsymbol{x}_k\cdot p[\cos(\boldsymbol{x}_k,\boldsymbol{w}_j)]^{p-1},\boldsymbol{\mu}_1\ran \\
        =&\; \lan \lp I-\frac{\boldsymbol{w}_j\boldsymbol{w}_j^T}{\|\boldsymbol{w}_j\|^2}\rp\sum_{k\leq K_1}\gamma_k(\boldsymbol{w}_j)y_k\boldsymbol{x}_k\cdot p[\cos(\boldsymbol{x}_k,\boldsymbol{w}_j)]^{p-1}, \boldsymbol{\mu}_1\ran\\
        =&\; p\cos^{p-1}(\boldsymbol{\mu}_1,\boldsymbol{w}_j)\lp 1-\cos^{2}(\boldsymbol{\mu}_1,\boldsymbol{w}_j)\rp +\sum_{l=2}^K\gamma_l(\boldsymbol{w}_j)y_lp\cos^{p}(\boldsymbol{\mu}_l,\boldsymbol{w}_j)\cos(\boldsymbol{w}_j,\boldsymbol{\mu}_1)\\
        =&\; \sqrt{1-\zeta}\lp p(1-\zeta)^{\frac{p-2}{2}}\zeta +\sum_{l=2}^K\gamma_l(\boldsymbol{w}_j)y_lp\cos^{p}(\boldsymbol{\mu}_l,\boldsymbol{w}_j)\rp\\
        \geq&\;\sqrt{1-\zeta}p\lp (1-\zeta)^{\frac{p-2}{2}}\zeta -\sum_{l=2}^K\lvt\lan \frac{\boldsymbol{w}_j}{\|\boldsymbol{w}_j\|},\boldsymbol{\mu}_l\ran\rvt^p\rp\\
        \geq &\; \sqrt{1-\zeta}p\lp (1-\zeta)^{\frac{p-2}{2}}\zeta -\lp \sum_{l=2}^K\lvt\lan \frac{\boldsymbol{w}_j}{\|\boldsymbol{w}_j\|},\boldsymbol{\mu}_l\ran\rvt^2\rp^{\frac{p}{2}}\rp\\
        \geq &\; \sqrt{1-\zeta}p\lp (1-\zeta)^{\frac{p-2}{2}}\zeta -\zeta^{\frac{p}{2}}\rp\geq \sqrt{1-\zeta}p\lp \lp 1-\frac{p-2}{2}\zeta\rp\zeta -\zeta^{\frac{p}{2}}\rp=p\zeta+o(\zeta)\,.
    \end{align*}
    Therefore, as long as $\zeta>0$ is small enough, which can be achieved by picking some $\delta<\bar{\delta}_2<1$, then $\inf_{\boldsymbol{w}_j\in\mathbb{S}^{D-1}, \cos\lp \boldsymbol{w}_j,\boldsymbol{\mu}_1\rp=1-\delta}\lan \mathcal{P}^\perp_{\boldsymbol{w}_j} x^{(1)}(\boldsymbol{w}_j),\boldsymbol{\mu}_1\ran = \Lambda_p(\delta)$ is positive. 
\end{proof}
\subsection{Proof for Auxiliary Lemmas}
\textbf{Balancedness}: Under GF, balancedness~\citep{Du&Lee} is preserved: $v_j^2(t)-\|\boldsymbol{w}_j(t)\|^2=0,\forall t\geq 0, \forall j\in[h]$, from the fact that:
\begin{align*}
    \frac{d}{dt}\|\boldsymbol{w}_j\|^2&=\;\lan \boldsymbol{w}_j, \dot{w}_j\ran\\
    &=\; -2\sum_{k=1}^K\gamma_k(\boldsymbol{w}_j)\nabla_{\hat{y}}\ell(y_k,f_p(\boldsymbol{x}_k;W,v)) v_j\lp \frac{p[\act(\lan \boldsymbol{x}_k,\boldsymbol{w}_j\ran)]^{p-1}}{\|\boldsymbol{w}_j\|^{p-1}}\lan \boldsymbol{w}_j,\boldsymbol{x}_k\ran-(p-1)\frac{[\act(\lan \boldsymbol{x}_k,\boldsymbol{w}_j\ran)]^p}{\|\boldsymbol{w}_j\|^{p+1}}\|\boldsymbol{w}_j\|^2\rp\\
    &=\;-2\sum_{k=1}^K\gamma_k(\boldsymbol{w}_j)\nabla_{\hat{y}}\ell(y_k,f_p(\boldsymbol{x}_k;W,v)) v_j\lp \frac{p[\act(\lan \boldsymbol{x}_k,\boldsymbol{w}_j\ran)]^{p-1}}{\|\boldsymbol{w}_j\|^{p}}-(p-1)\frac{[\act(\lan \boldsymbol{x}_k,\boldsymbol{w}_j\ran)]^p}{\|\boldsymbol{w}_j\|^{p-1}}\rp\\
    &=\;-2\sum_{k=1}^K\gamma_k(\boldsymbol{w}_j)\nabla_{\hat{y}}\ell(y_k,f_p(\boldsymbol{x}_k;W,v)) v_j\frac{[\act(\lan \boldsymbol{x}_k,\boldsymbol{w}_j\ran)]^p}{\|\boldsymbol{w}_j\|^{p-1}}\\
    &=\;\frac{d}{dt}v_j^2
\end{align*}

In addition, $\sign(v_j(t))=\sign(v_j(0)), \forall t\geq 0,\forall j\in[h]$, and the dynamical behaviors of neurons will be divided into two types, depending on $\sign(v_j(0))$. Therefore, throughout the gradient flow trajectory, we have $v_j=\sign(v_j(0))\|\boldsymbol{w}_j\|$. This fact will be used in the subsequent proof.

\begin{proof}[Proof for Lemma \ref{lem_small_norm}]
    Under gradient flow, we have
    \be
        \frac{d}{dt}\boldsymbol{w}_j=-\sum_{k=1}^K\gamma_k(\boldsymbol{w}_j)\nabla_{\hat{y}}\ell(y_k,f_p(\boldsymbol{x}_k;W,v)) v_j\lp \frac{p[\act(\lan \boldsymbol{x}_k,\boldsymbol{w}_j\ran)]^{p-1}}{\|\boldsymbol{w}_j\|^{p-1}}\boldsymbol{x}_k-(p-1)\frac{[\act(\lan \boldsymbol{x}_k,\boldsymbol{w}_j\ran)]^p}{\|\boldsymbol{w}_j\|^{p+1}}\boldsymbol{w}_j\rp\,.
    \ee
    and for $\|\boldsymbol{w}_j\|$,
    \begin{align*}
    \frac{d}{dt}\|\boldsymbol{w}_j\|^2&=\;\lan \boldsymbol{w}_j, \dot{w}_j\ran\\
    &=\; -2\sum_{k=1}^K\gamma_k(\boldsymbol{w}_j)\nabla_{\hat{y}}\ell(y_k,f_p(\boldsymbol{x}_k;W,v)) v_j\lp \frac{p[\act(\lan \boldsymbol{x}_k,\boldsymbol{w}_j\ran)]^{p-1}}{\|\boldsymbol{w}_j\|^{p-1}}\lan \boldsymbol{w}_j,\boldsymbol{x}_k\ran-(p-1)\frac{[\act(\lan \boldsymbol{x}_k,\boldsymbol{w}_j\ran)]^p}{\|\boldsymbol{w}_j\|^{p+1}}\|\boldsymbol{w}_j\|^2\rp\\
    &=\;-2\sum_{k=1}^K\gamma_k(\boldsymbol{w}_j)\nabla_{\hat{y}}\ell(y_k,f_p(\boldsymbol{x}_k;W,v)) v_j\lp \frac{p[\act(\lan \boldsymbol{x}_k,\boldsymbol{w}_j\ran)]^{p-1}}{\|\boldsymbol{w}_j\|^{p}}-(p-1)\frac{[\act(\lan \boldsymbol{x}_k,\boldsymbol{w}_j\ran)]^p}{\|\boldsymbol{w}_j\|^{p-1}}\rp\\
    &=\;-2\sum_{k=1}^K\gamma_k(\boldsymbol{w}_j)\nabla_{\hat{y}}\ell(y_k,f_p(\boldsymbol{x}_k;W,v)) v_j\frac{[\act(\lan \boldsymbol{x}_k,\boldsymbol{w}_j\ran)]^p}{\|\boldsymbol{w}_j\|^{p-1}}
\end{align*}
    Balanced initialization enforces $v_j=\sign(v_j(0))\|\boldsymbol{w}_j\|$, hence
    \be
        \frac{d}{dt}\|\boldsymbol{w}_j\|^2=-2\sum_{k=1}^K\gamma_k(\boldsymbol{w}_j)\nabla_{\hat{y}}\ell(y_k,f_p(\boldsymbol{x}_k;W,v)) \sign(v_j(0))\|\boldsymbol{w}_j\|^2\frac{[\act(\lan \boldsymbol{x}_k,\boldsymbol{w}_j\ran)]^p}{\|\boldsymbol{w}_j\|^{p}}\,.
    \ee
    Let $T:=\inf\{t:\ \max_{i}|f(\boldsymbol{x}_k;W(t),v(t))|>2\epsilon \sqrt{h}M^2\}$,  then $\forall t\leq T, j\in[h]$, we have 
    \begin{align}
        \frac{d}{dt}\|\boldsymbol{w}_j\|^2
        &=\;-2\sum_{k=1}^K\gamma_k(\boldsymbol{w}_j)\nabla_{\hat{y}}\ell(y_k,f_p(\boldsymbol{x}_k;W,v)) \sign(v_j(0))\|\boldsymbol{w}_j\|^2\frac{[\act(\lan \boldsymbol{x}_k,\boldsymbol{w}_j\ran)]^p}{\|\boldsymbol{w}_j\|^{p}}\nonumber\\
        &=\; -2\sum_{k=1}^K\gamma_k(\boldsymbol{w}_j)\nabla_{\hat{y}}\ell(y_k,f(\boldsymbol{x}_k;W,v))\sign(v_j(0))\|\boldsymbol{w}_j\|^2\frac{(\lan \boldsymbol{x}_k,\boldsymbol{w}_j\ran)^p}{\|\boldsymbol{w}_j\|^{p}}\nonumber\\
        &\leq\;2\sum_{k=1}^K\lvt\nabla_{\hat{y}}\ell(y_k,f(\boldsymbol{x}_k;W,v))\rvt\|\boldsymbol{w}_j\|^2\nonumber\\
        &\leq\; 2\sum_{k=1}^K(|y_k|+2|f(\boldsymbol{x}_k;W,v)|)  \|\boldsymbol{w}_j\|^2\nonumber\\
        &\leq\; 2\sum_{k=1}^K(1+4\epsilon \sqrt{h}M^2)   \|\boldsymbol{w}_j\|\nonumber\\
        &\leq \; 2n(+4\epsilon\sqrt{h}M^2 )\|\boldsymbol{w}_j\|^2\,.&
    \end{align}
    Let $\tau_j:=\inf\{t: \|\boldsymbol{w}_j(t)\|^2>\frac{2\epsilon M^2}{\sqrt{h}}\}$, and let $j^*:=\arg\min_j \tau_j$, then $\tau_{j^*}=\min_{j}\tau_j\leq T$ due to the fact that 
    \ben
        |f(x_{i};W,v)|= \lvt\sum_{j\in[h]}\one_{\lan \boldsymbol{w}_j,\boldsymbol{x}_k\ran>0}v_j\frac{(\lan \boldsymbol{w}_j,\boldsymbol{x}_k\ran)^p}{\|\boldsymbol{w}_j\|^p}\rvt\leq \sum_{j\in[h]} \|\boldsymbol{w}_j\|^2\leq h\max_{j\in[h]}\|\boldsymbol{w}_j\|^2\,,
    \een
    which implies "$|f(\boldsymbol{x}_k;W(t),v(t))|>2\epsilon \sqrt{h}M^2\Rightarrow \exists j, s.t.\|\boldsymbol{w}_j(t)\|^2>\frac{2\epsilon M^2}{\sqrt{h}}$". 
    
    Then for $t\leq \tau_{j^*}$, we have
    \be
        \frac{d}{dt}\|w_{j^*}\|^2\leq 2n(+4\epsilon\sqrt{h}M^2 )\|w_{j^*}\|^2\,.
    \ee
    By Gr\"onwall's inequality, we have $\forall t\leq \tau_{j^*}$
    \begin{align*}
        \|w_{j^*}(t)\|^2&\leq\; \exp\lp 2n(+4\epsilon\sqrt{h}M^2 )t\rp\|w_{j^*}(0)\|^2\,,\\
         &=\;\exp\lp 2n(+4\epsilon\sqrt{h}M^2 )t\rp\epsilon^2\|[W_0]_{:,j^*}\|^2\\
         &\leq \;\exp\lp 2n(+4\epsilon\sqrt{h}M^2 )t\rp\epsilon^2M^2\,.
    \end{align*}
    Suppose $\tau_{j^*}<\frac{1}{4n}\log\lp\frac{1}{\sqrt{h}\epsilon}\rp$, then by the continuity of $\|w_{j^*}(t)\|^2$, we have 
    \begin{align*}
        \frac{2\epsilon M^2}{\sqrt{h}}\leq \|w_{j^*}(\tau_{j^*})\|^2&\leq\; \exp\lp 2n(+4\epsilon\sqrt{h}M^2 )\tau_{j^*}\rp\epsilon^2M^2\\
        &\leq\; \exp\lp 2n(+4\epsilon\sqrt{h}M^2 )\frac{1}{4n}\log\lp\frac{1}{\sqrt{h}\epsilon}\rp\rp\epsilon^2M^2\\
        &\leq\; \exp\lp 
        \frac{1+4\epsilon \sqrt{h}M^2}{2}\log\lp\frac{1}{\sqrt{h}\epsilon}\rp\rp\epsilon^2M^2\\
        &\leq\; \exp\lp 
        \log\lp\frac{1}{\sqrt{h}\epsilon}\rp\rp\epsilon^2M^2=\frac{\epsilon M^2}{\sqrt{h}}\,,
    \end{align*}
    which leads to a contradiction $2\epsilon\leq \epsilon$. Therefore, one must have $T\geq \tau_{j^*}\geq \frac{1}{4n}\log\lp\frac{1}{\sqrt{h}\epsilon}\rp$. This finishes the proof.
\end{proof}
\begin{proof}[Proof of Lemma \ref{lem_sc}]
    Since 
    \be
    g_p'(q;\{z_i\}_{i=1}^n)=\lp\sum_{k=1}^Kz_i^q\log z_i\rp\lp\sum_{k=1}^Kz_i^{p+1-q}\rp-\lp\sum_{k=1}^Kz_i^q\rp\lp\sum_{k=1}^Kz_i^{p+1-q}\log z_i\rp\,,
    \ee
    we immediately find $g_p'(q^*;\{z_i\}_{i=1}^n)=0$. Now we compute the second-order derivative
    \begin{align*}
            g_p''(q;\{z_i\}_{i=1}^n)&\;=-2\lp\sum_{k=1}^Kz_i^q\log z_i\rp\lp\sum_{k=1}^Kz_i^{p+1-q}\log z_i\rp\\
    &\;\quad\quad+\lp\sum_{k=1}^Kz_i^q\rp\lp\sum_{k=1}^Kz_i^{p+1-q}\log^2 z_i\rp+\lp\sum_{k=1}^Kz_i^q\log^2 z_i\rp\lp\sum_{k=1}^Kz_i^{p+1-q}\rp\\
    &=\; \sum_{1\leq i,j\leq n}z_i^qz_j^{p+1-q}(-2\log z_i\log z_j)\\
    &\;\quad\quad+\sum_{1\leq i,j\leq n}z_i^qz_j^{p+1-q}\log^2 z_j+\sum_{1\leq i,j\leq n}z_i^qz_j^{p+1-q}\log^2 z_i\\
    &=\; \sum_{1\leq i,j\leq n}z_i^qz_j^{p+1-q} (\log z_i-\log z_j)^2\geq 0\,,
    \end{align*}
    and the equality holds only when $z_1=\cdots=z_n$. The desired results follow.
\end{proof}

\end{document}